%% file: 0-main.tex
\documentclass[letterpaper]{article} 
\usepackage{aaai24}  
\usepackage{times}  
\usepackage{helvet}  
\usepackage{courier}  
\usepackage[hyphens]{url}  
\usepackage{graphicx} 
\usepackage{natbib}  
\usepackage{caption} 
\usepackage{algorithm}
\usepackage{algorithmic}

\usepackage{newfloat}
\usepackage{listings}
\DeclareCaptionStyle{ruled}{labelfont=normalfont,labelsep=colon,strut=off} 
\floatstyle{ruled}
\newfloat{listing}{tb}{lst}{}
\floatname{listing}{Listing}
\title{Fine-tuning Graph Neural Networks by Preserving Graph Generative Patterns}
\author{
    Yifei Sun\textsuperscript{\rm 1},
    Qi Zhu\textsuperscript{\rm 2},
    Yang Yang\textsuperscript{\rm 1}\thanks{Corresponding author.},
    Chunping Wang\textsuperscript{\rm 3},
    Tianyu Fan\textsuperscript{\rm 1},
    Jiajun Zhu\textsuperscript{\rm 1},
    Lei Chen\textsuperscript{\rm 3}
}
\affiliations{
    \textsuperscript{\rm 1}Zhejiang University, Hangzhou, China \\
    \textsuperscript{\rm 2}University of Illinois Urbana-Champaign, USA \\
    \textsuperscript{\rm 1}FinVolution Group, Shanghai, China\\
    \{yifeisun, yangya, fantianyu, junnian\}@zju.edu.cn, qiz3@illinois.edu, \{wangchunping02, chenlei04\}@xinye.com

}

\usepackage{bibentry}

\input{0.5-usepacks}

\begin{document}

\maketitle

\input{1-abs}
\input{2-intro}

\input{3-prelim}
\input{4-model}
\input{5-exp}
\input{6-conclusion}

\clearpage
\section*{Acknowledgments}
This work is supported by NSFC (No.62322606) and the Fundamental Research Funds for the Central Universities.

\bibliography{aaai24}

\clearpage 
\onecolumn{\input{7-appendix}} 

\end{document}

%% file: 0.5-usepacks.tex
\usepackage[utf8]{inputenc} 
\usepackage[T1]{fontenc}    
\usepackage{url}            
\usepackage{booktabs}       
\usepackage{amsfonts}       
\usepackage{nicefrac}       
\usepackage{microtype}      
\usepackage{xcolor}         
\usepackage{xspace}         

\usepackage{graphicx} 
\usepackage{amsthm}
\usepackage{amsmath}
\usepackage{amssymb}
\usepackage{paralist}
\usepackage{tabularx}
\usepackage{bm}

\usepackage{color,soul}
\usepackage{float} 

\usepackage{threeparttable}
\usepackage{multirow}

\newcommand{\ie}{\emph{i.e.}\xspace} 
\newcommand{\etc}{\emph{etc.}\xspace} 
\newcommand{\wrt}{\emph{w.r.t.}\xspace} 
\def \B {\mathcal{B}}

\newcommand{\ssym}{§}

\newcommand{\vpara}[1]{\vspace{0.05in}\noindent\textbf{#1 }}
\newcommand{\model}{\textbf{\textsc{G-Tuning}}\xspace} 

\newcommand{\Wencoder}{graphon encoder\xspace}

\newtheorem{definition}{Definition}
\newtheorem{theorem}{Theorem}

\newtheorem{lemma}[theorem]{Lemma}
\newtheorem{thm:eg}{Example}

%% file: 1-abs.tex
\begin{abstract}
Recently, the paradigm of pre-training and fine-tuning graph neural networks has been intensively studied and applied in a wide range of graph mining tasks. 
Its success is generally attributed to the structural consistency between pre-training and downstream datasets, which, however, does not hold in many real-world scenarios. 
Existing works have shown that the \textit{structural divergence} between pre-training and downstream graphs significantly limits the transferability when using the vanilla fine-tuning strategy. This divergence leads to model overfitting on pre-training graphs and causes difficulties in capturing the structural properties of the downstream graphs. 
In this paper, we identify the fundamental cause of structural divergence as the discrepancy of \textit{generative patterns} between the pre-training and downstream graphs.
Furthermore, we propose \model to preserve the generative patterns of downstream graphs. 
Given a downstream graph $G$, the core idea is to tune the pre-trained GNN so that it can reconstruct the \textit{generative patterns} of $G$, the graphon $W$.
However, the exact reconstruction of a graphon is known to be computationally expensive. To overcome this challenge, we provide a theoretical analysis that establishes the existence of a set of alternative graphons called graphon bases for any given graphon. By utilizing a linear combination of these graphon bases, we can efficiently approximate $W$. This theoretical finding forms the basis of our proposed model, as it enables effective learning of the graphon bases and their associated coefficients.
Compared with existing algorithms,
\model demonstrates an average improvement of 0.5\% and 2.6\% on in-domain and out-of-domain transfer learning experiments, respectively.
    
\end{abstract}

%% file: 2-intro.tex
 \section{Introduction} \label{sec:intro}
The development of graph neural networks (GNNs) has revolutionized many tasks of various domains in recent years. However, labeled data is extremely scarce due to the time-consuming and laborious labeling process. To address this obstacle, the ``pre-train and fine-tune'' paradigm has made substantial progress~\cite{xia2022survey,li2022kpgt,jiao2022energymotivated} and attracted considerable research interests. Specifically, this paradigm involves pre-training a model on a large-scale graph dataset, followed by fine-tuning its parameters on downstream graphs by specific tasks.

\begin{figure*}[t!]
    \centering
    \setlength{\abovecaptionskip}{-0.05cm} 
    \includegraphics[width=\textwidth]{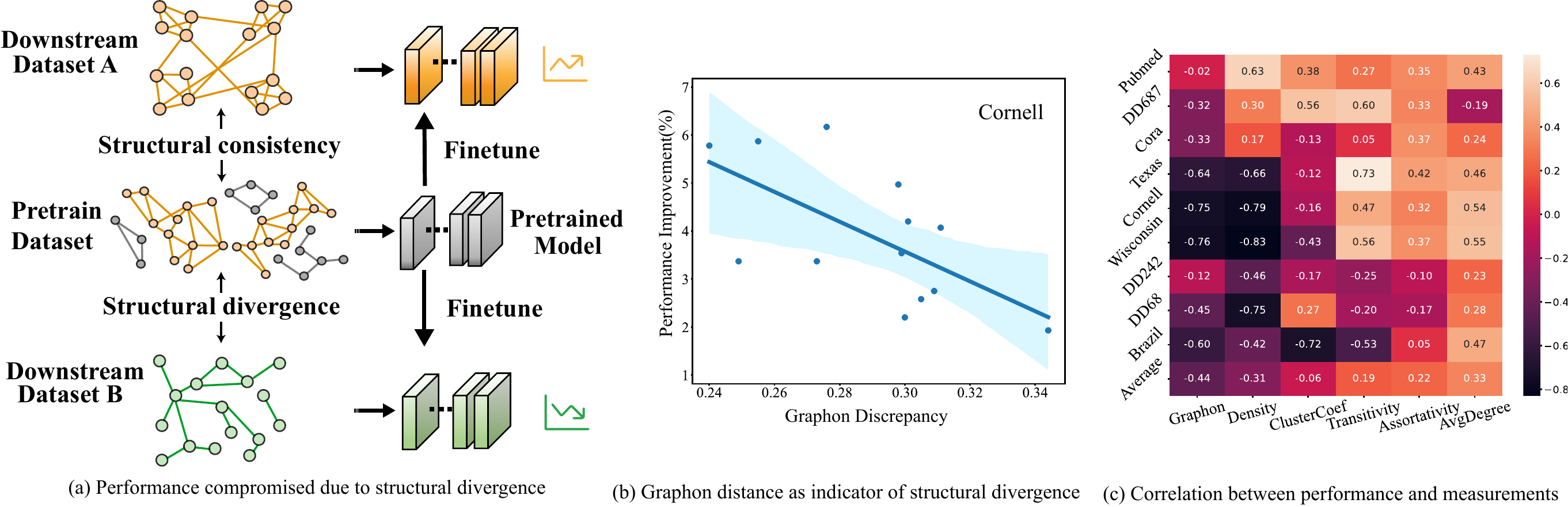}
    \setlength{\belowcaptionskip}{-0.45cm}
    \caption{The sketch and the observations of the structural divergence. (a) The sketch shows the performance under different scenarios. (b) The case study shows that the larger the graphon Gromov-Wasserstain (GW) discrepancy between different pre-train datasets (shown in App \ssym ~A.4) and downstream dataset Cornell~\protect\cite{cornell} is, the less the performance (\%) is promoted. (c) The pearson correlation between: (Horizontal) discrepancy of different graph measurements and graphon between pre-training and downstream datasets, (Vertical) the performance promotion for different downstream datasets when using different pre-train datasets.}
    \label{fig:intro}
\end{figure*}

The success of ``pre-train and fine-tune'' paradigm is generally attributed to the structural consistency between pre-training and downstream graphs~\cite{hu2020gpt,hu2020strategies,qiu2020gcc,sun2021mocl,xu2023better}. 
However, in real-world scenarios, structural patterns vary dramatically across different graphs, and patterns in the downstream graphs may not be readily available in the pre-training graph. In the context of molecular graphs, a well-known out-of-distribution problem arises when the training and testing graphs originate from different environments, characterized by variations in \emph{size} or \emph{scaffold}. As a consequence, the downstream molecular graph data often encompasses numerous novel substructures that has not been encountered during training.
Hence, structural consistency does not always hold. 
Fig~\ref{fig:intro}(a) shows that while structural consistency (shown in orange) between the pre-training and downstream dataset A ensures the promotion of performance, the \textit{structural divergence} (shown in green) causes a degradation of performance when fine-tuned on downstream dataset B. In some cases, it can even lead to worse results than those obtained without pre-training. 

In light of this, we are intrigued by the relationship between structural divergence and the extent of performance improvements on downstream graphs.
Specifically, \textit{graphon} is a well-known non-parametric function on graph that has been proved to effectively describe the generative mechanism of graphs (\emph{i.e.,} \textit{generative patterns})~\cite{lovasz2012large}.
In Fig~\ref{fig:intro}(b), we calculate the Gromov-Wasserstain (GW) discrepancy (a distance metric between geometry objects) of graphons between different pre-training and one test graph and report the corresponding performance on the same downstream graph.
Interestingly, as the difference between graphons of pre-training dataset and that of downstream dataset increases, the performance improvement diminishes.
To further validate this, we compute the Pearson CC (Correlation Coefficient) between other representative graph measurements (e.g., density, transitivity, \emph{etc.}) and the performance improvement.
As Fig~\ref{fig:intro}(c) suggests, most of them cannot reflect the degree of performance improvement, and only the graphon discrepancy is consistently negatively correlated with the degree of performance promotion. 
Hence, we attribute the subpar performance of fine-tuning to the disparity in the \textit{generative patterns} between the pre-training and fine-tuning graphs.
Nevertheless, fine-tuning with respect to \textit{generative patterns} poses significant challenges: (1) the structural information of the pre-training may not be accessible during fine-tuning, and (2) effectively representing intricate semantics within these \textit{generative patterns}, such as graphons, requires careful design considerations.

In this paper, we aim to address these challenges by proposing a fine-tuning strategy, \model, which is agnostic to pre-training data and algorithms.
Specifically, it performs graphon reconstruction of the downstream graphs during fine-tuning.
In order to enable efficient reconstruction, we provide a theoretical result (Theorem~\ref{theo:taylor}) that given a graphon $W$, it's possible to find a set of other graphons, called graphon bases, whose linear combination can closely approximate $W$. 
Then, we develop a graphon decoder that transforms the embeddings from the pre-trained model into a set of coefficients. These coefficients are combined with the structure-aware learnable bases to form the reconstructed graphon. To ensure the fidelity of the reconstructed graphon, we introduce a GW-discrepancy based loss, which minimizes the distance between the approximated graphon and an oracle graphon ~\cite{xu2021learning}. 
Furthermore, by optimizing our proposed \model, we obtain provable results regarding the discriminative subgraphs relevant to the task (Theorem~\ref{theo:discriminative}).

The main contributions of our work are as follows:\footnote{Supplement materials:~\url{https://github.com/zjunet/G-Tuning}}
\begin{compactitem}   
\item We identify the generative patterns of downstream graphs as a crucial step in bridging the gap between pre-training and fine-tuning.

\item Building upon on our theoretical results, we design the model architecture, \model to efficiently reconstruct graphon as generative patterns with rigorous generalization results.

\item Empirically, our method shows an average of 0.47\% and 2.62\% improvement on 8 in-domain and 7 out-of-domain transfer learning datasets over the best baseline.  
\end{compactitem}

%% file: 3-prelim.tex
\section{Preliminaries} \label{sec:prelim}
\vpara{Notations.}
Let $\mathcal{G} =\left(V, A, X\right)$ denote a graph, where $V$ is the node set, $A\in \{0,1\}^{|{V}|\times |{V}|}$ is the adjacency matrix and $X\in \mathbb{R}^{|{V}|\times d}$ is the node feature matrix where $d$ is the dimension of feature.
$\mathcal{G}_s$ and $\mathcal{G}_t$ denotes a pre-training graph and a downstream graph respectively.
The classic and commonly used pre-training paradigm is to first pre-train the backbone $\Phi$ on abundant unlabeled graphs by a self-supervised task with the loss as $\mathcal{L}_\mathrm{SSL}$.
Then the pre-trained $\Phi$ is employed to fine-tuning on labeled downstream graphs. The embedding ${H} \in \mathbb{R}^{|{V}|\times {d}}$ with the hidden dimension $d$ encoded by $\Phi$ is further input to a randomly initialized task-specific shallow model $f_{\mathbf{\phi}}$. The goal of fine-tune is to adapt both $\Phi$ and $f_{\mathbf{\phi}}$ for the downstream task with the loss $\mathcal{L}_\mathrm{task}$ and the label ${Y}$. 
The pre-training setup has two variations: one is the in-domain setting, where $\mathcal{G}_s$ and $\mathcal{G}_t$ come from the same domain, and the other is the out-of-domain setting, where $\mathcal{G}_s$ and $\mathcal{G}_t$ originate from different domains. In the latter case, the structural divergence between $\mathcal{G}_s$ and $\mathcal{G}_t$ is greater, posing greater challenges for fine-tuning.

\begin{definition}[Graph generative patterns]
Graph generative patterns are data distributions parameterized by $\Theta$ where the observed graphs $\{G_1,...,G_n\}$ are sampled from, $G_i \sim P({G};\Theta)$.
\end{definition}
According to our definition, $\Theta$ can be traditional graph generative model such as Erdos-Rényi~\citep{erdHos1959random}, stochastic block model~\citep{airoldi2008mixed}, forest-fire graph~\citep{leskovec2007graph} and \etc
Besides, $\Theta$ can also be any deep generative model like GraphRNN~\citep{you2018graphrnn}. In this paper, we propose a theoretically sound fine-tuning framework by preserving the graph generative pattern of the downstream graphs.

\vpara{Graphon.}  
A graphon~\cite{airoldi2013stochastic}, short for "graph function", can be interpreted as a generalization of a graph with an uncountable number of nodes or a graph generative model or, more important for this work, a mathematical object representing $\Theta$ from graph generative patterns $P({G};\Theta)$. 
Formally, a graphon is a continuous and symmetric function $W : [0, 1]^2 \rightarrow [0, 1]$. Given two points $u_i, u_j \in [0,1]$ as ``nodes'', ${W}(i,j) \in [0,1]$ indicates the probability of them forming an edge. 
The main idea of graphon is that when we extract subgraphs from the observed graph, the structure of these subgraphs becomes increasingly similar to that of the observed one as we increase the size of subgraphs. The structure then converges in some sense to a limit object, graphon. 
The convergence is defined via the convergence of homomorphism densities.
Homomorphism density $t(F,G)$ is used to measure the relative frequency that homomorphism of graph $F$ appears in graph $G$: 
$t(F, G)=\frac{|\operatorname{hom}(F, G)|}{\left|V_G\right|^{\left|V_H\right|}}$
, which can be seen as the probability that a random mapping of vertices from $F$ to $G$ is a homomorphism. Thus, the convergence can be formalized as $\lim\limits_{n \to \infty} {t\left(F, G_n\right)}=t(F, W)$.
When used as the graph generative patterns, the adjacency matrix $A$ of graph $\mathcal{G}$ with $N$ nodes are sampled from $P({G};W)$ as follows,
\begin{equation}
    v \sim \mathbb{U}(0,1), v \in V ; A_{ij} \sim \text{Ber}(W(v_i, v_j)), \forall i,j \in [N].
\end{equation}
where Ber$(\cdot)$ is the Bernoulli distribution. 
Since there is no available closed-form expression of graphon, existing works mainly employ a two-dimensional step function, which can be seen as a matrix, to represent a graphon~\cite{xu2021learning,han2022g}. In fact, the weak regularity lemma of graphon~\cite{lovasz2012large} indicates that an arbitrary graphon can be approximated well by a two-dimensional step function. Hence, we follow the above mentioned works to employ a step function $W \in [0,1]^{D \times D}$ to represent a graphon, where $D$ is a hyper-parameter. 

\section{Related Work} \label{sec:related}

\vpara{Fine-tuning strategies.}
Designing fine-tuning strategies first attracts attention in computer vision (CV), which can be categorized into model parameter regularization and feature regularization. 
L2\_SP\cite{xuhong2018explicit} uses $L^2$ distance to constrain the parameters around pre-trained ones.
StochNorm~\cite{kou2020stochastic} replaces BN(batch normalization) layers in pre-trained model with their StochNorm layers.
DELTA~\cite{li2018delta} selects features with channel-wise attention to constrain. BSS~\cite{chen2019catastrophic} penalizes small eigenvalues of features to prevent negative transfer.
However, 
there is only one work focusing on promoting performance of downstream task during fine-tuning phase specially for the graph structured data. 
GTOT-Tuning~\cite{ijcai2022p518} presents an optimal transport-based feature regularization, which achieves node-level transport through graph structure. 
Moreover, the gap between pre-train and fine-tune is also noted in L2P~\cite{lu2021learning} and AUX-TS~\cite{han2021adaptive}. 
Specifically, L2P leverages meta-learning to adjust tasks during pre-training stage. AUX-TS adaptively selects and combines auxiliary tasks with the target task in fine-tuning stage, which means only one pre-train task is not enough. 
However, they both require to use the same dataset for both pre-train and fine-tune, and to insert auxiliary tasks in pre-training phase , which indicates they are not generally applicable.
Unlike them, we focus on the fine-tuning phase and propose that preserving generative patterns of downstream graphs during fine-tuning is the key to mine knowledge from pre-trained GNNs without altering the pre-training process.

\vpara{Graphon.}
Over the past decade, graphon has been studied intensively as a mathematical object~\cite{borgs2008convergent,borgs2012convergent,lovasz2006limits,lovasz2012large} and been applied broadly, like graph signal processing~\cite{ruiz2020fourier,ruiz2021graphon}, game theory~\cite{parise2019graphon}, network science~\cite{avella2018centrality,vizuete2021laplacian}.
Moreover, G-mixup~\cite{han2022g} is proposed to conduct data augmentation for graph classification since a graphon can serve as a graph generator. From the another perspective of being the graph limit, \cite{ruiz2020graphon} leverage graphon to analyse the transferability of GNNs.
Graphon, as limit of graphs, forms a natural method to describe graphs and encapsulates the generative patterns ~\cite{borgs2017graphons}.
Thus, we incorporate the graphon into the fine-tuning stage to preserve the generate patterns of downstream graphs.
In App \ssym~A.1, we provide more detailed related work about graph pre-training and graphon.

%% file: 4-model.tex
\begin{figure*}[t!]
    \centering
    \includegraphics[width=.90\textwidth]{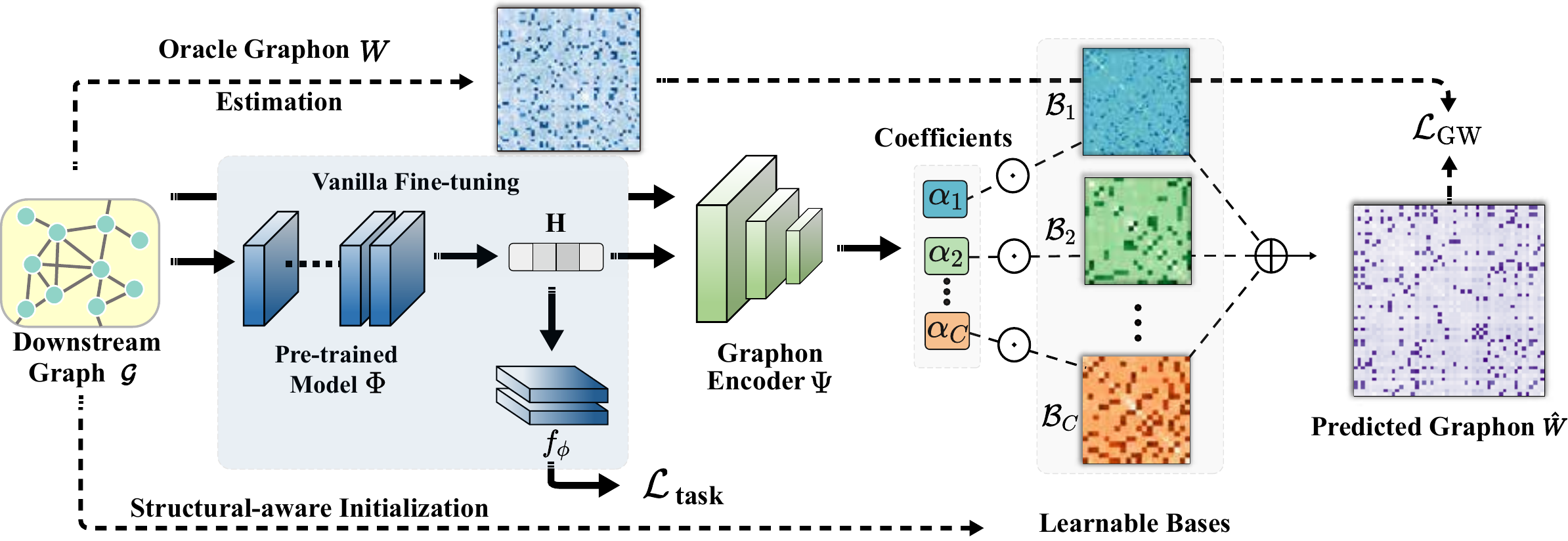}
    \setlength{\belowcaptionskip}{-0.3cm} 
    \caption{Overall workflow of \model. Each basis $\B_i$ and predicted graphon $\hat{W}$ is produced during actual training.}
    \label{fig:model}
\end{figure*}

\section{\model} \label{sec:model}

\subsection{Framework Overview}\label{subsec:general}
Similar to our seminal results in Fig~\ref{fig:intro}, recent research~\cite{zhu2021transfer} has started to analyze the transfer performance of a pre-trained GNN \wrt discrepancy between pre-train and fine-tune graphs. 
However, the pre-training graphs are typically not accessible during the fine-tuning phase. Thus, \model aims to adapt the pre-trained GNN to the fine-tuning graphs by preserving the generative patterns.
During fine-tuning, a pre-trained GNN $\Phi$ obtains latent representations $H$ for downstream graphs $\mathcal{G}_t = \{G_1,...,G_n\}$ and feeds them into task-specific layers $f_\phi$ to train with the fine-tuning labels $Y$. 
For a specific graph $G_i(A,X)$, the pre-trained node embeddings $H$ are obtained by pre-trained model $\Phi$: 
\begin{equation} \label{eq:standard_gnn}
    \mathcal{L}_\text{task} = \mathcal{L}_\text{CE}(f_{\phi}(H),{Y}), \quad \ H = \Phi(A, X),
\end{equation}
\noindent where $\Phi$ can be any pre-training backbone model, $\mathcal{L}_\text{CE}$ is the cross entropy classification loss and $f_\phi$ is a shallow neural network $f_\phi: H \rightarrow \hat{Y}$.

However, the vanilla strategy may fail to improve fine-tuning performance due to the large discrepancy between pre-training and fine-tuning graphs, namely \emph{negative transfer}. 
To alleviate this, we propose to enable the pre-trained GNN $\Phi$ to preserve the generative patterns of the downstream graphs $\mathcal{G}_t$ by reconstructing their graphons $W$ (overall workflow in Fig~\ref{fig:model}).
However, the pre-trained model is inherently biased towards the generative patterns of the pre-training graphs. Consequently, at the begining of the fine-tuning, the embeddings $H$ also contain the bias from the pre-training data. Thus, we need both the $H$ and graph structure $A$ of downstream graph to reconstruct the graphon.
Specifically, we design a graphon reconstruction module $\Omega$ to reconstruct the graphon $\hat{W}$ of each downstream graph $G_i(A,X) \in \mathcal{G}_t$:
\begin{equation} \label{eq:total_rec}
    {\hat{W}}={\Omega}({A}, {H}),
\end{equation}
However, the computation of oracle graphons $W$ of real-world graphs in continuous function form itself is intractable~\cite{han2022g}. Thus, the graphon reconstruction module $\Omega$ approximates an estimated oracle graphon (\ie $W \in[0,1]^{D \times D}$) for each downstream graph through $\mathcal{L}_\text{aux}$, D is the size of oracle graphon.
Finally, in the framework of \model (Fig~\ref{fig:model}), we leverage both the downstream task loss and our reconstruction loss to optimize the parameters of pre-trained GNN encoder $\Phi$, the layer $f_\phi$ and graphon reconstruction module $\Omega$ as follows,
\begin{equation} \label{Eq:total_loss}
    \mathcal{L} = \mathcal{L}_\text{task} + \lambda  \mathcal{L}_\model(W, \hat{W}).
\end{equation}
\noindent where $\lambda$ is a hyper-parameter of \model.

\subsection{Approximating Oracle Graphon} 
In this section, we first discuss the calculation of the underlying graphon, denoted as $W$, given downstream graphs $\mathcal{G}_t$. Then we introduced proposed algorithm to reconstruct this graphon as an auxiliary loss during the fine-tuning process.

Extensive research has been conducted on the methods for estimating graphons~\cite{airoldi2013stochastic,chatterjee2015matrix,pensky2019dynamic,ruiz2020graphon} from observed graphs.
In this paper, the oracle graphon $W$ used in Eq~\ref{Eq:total_loss} can be estimated via structured Gromov-Wasserstein barycenters (SGB)~\cite{xu2021learning}. Suppose there are N observed graphs $\{G_n\}_{n=1}^N$ and their adjacency matrices $\{A_n\}_{n=1}^N$, at each step $t$, 
\begin{eqnarray}
\min_{\mathbf{T}_n\in\Pi({\mu},{\mu}_W)} d_{{gw}, 2}^2\left({W_t}, {A_n}\right), \\
{W_{t+1}} = \frac{1}{\mu_W \mu_W^{\top}} \sum_{n=1}^N  \mathbf{T_n}^{\top} {A_n} \mathbf{T_n}, \label{eq:prox}
\end{eqnarray}
where $W_{t+1} \in [0,1]^{D \times D}$ is calculated barycenters with the optimal transportation plan $\{\mathbf{T_n}\}_{n=1}^{N}$ of 2-order Gromov-Wasserstein distance $d_{\mathrm{gw}, 2}^2$. The probability measures ${\mu}_W$ is estimated by merging and sorting the normalized node degrees in $\{A_n\}_{n=1}^N$. Further details regarding the calculations of $d_{\mathrm{gw}, 2}^2$ and the iterative procedure for estimating the oracle graphon can be found in App \ssym~A.3.

After obtaining the oracle graphon $W \in [0,1]^{D \times D}$ from the downstream graphs and the reconstructed graphon $\hat{W} \in [0,1]^{M \times M}$ from the reconstruction module $\Omega$, we also adopt the Gromov-Wasserstein distance to measure the distance between the two unaligned graphons.
Formally, our $p$-order GW distance is calculated as:
\begin{equation}
\label{Eq:loss}
 \mathcal{L}_\model(W, \hat{W})= \min\limits_{T\in \Gamma}\sum_{i,j,k,l}(W_{i,k}-\hat{W}_{j,l})^p T_{i,j} T_{k,l}, 
\end{equation}
\noindent where $\Gamma$ is the set of transportation plans that satisfy $\Gamma=\{T \in \mathbb{R}_{+}^{D \times M} | T 1_M= 1_D, T^\intercal 1_D=1_M \}$. In each epoch, an optimal $T$ minimizes the transportation cost between two graphons $(W, \hat{W})$. Fixing $T$, the graphon reconstruction module $\Omega$ is optimized to minimize the GW discrepancy. We will introduce the model design of $\Psi$ in the next section for scalable training and inference.

\subsection{Efficient Optimization of $\mathcal{L}_\model$} \label{sec:Optimization}
A straightforward way to approximate the graphon is to learn a mapping function graph structure $A$ and node embedding $H$ to the target $W$. For example, we can simply define the graphon reconstruction module $\Omega$ in Eq~\ref{eq:total_rec} as a GNN or a shallow MLP. 
Suppose we have $M$ graphs and each graph has at most $|V|$ nodes, brute-force graphon reconstruction from the pre-trained node embeddings $H \in \mathbb{R}^{|V| \times d}$ requires a large number of parameters, e.g., $\Psi: \mathbb{R}^{|V| \times d} \rightarrow \mathbb{R}^{M \times M}$. Additionally, properties such as permutation-variance of graphons are not guaranteed without a carefully designed architecture.

To this end, we first establish a theorem for graphon decomposition and utilize it for efficient graphon approximation. 
Specifically, we propose that any graphon can be reconstructed by a linear combination of graphon bases $\B_k \in \boldsymbol{\B}$.

\begin{theorem}\label{theo:taylor}
$\forall \ W(x, y)\in\mathcal{W}_{C + 1}$,  there exists $C$ graphon bases ${\B_k(x, y)}$ that satisfies $W(x, y) = \sum^{C}_{k = 1}\alpha_k \B_k(x, y) + R_{C+1}(x, y)$, where $\alpha_i\in \mathbb{R} $ and $R_{C+1}(x, y)$ is the remainder of order $C+1$.
\begin{eqnarray}
W(x,y) =&\sum^{C}_{k=0} \alpha_k \B_k(x,y) +R_{C+1}.
\end{eqnarray}
\end{theorem}
For every $W \in \mathcal{W}_{C + 1}$, it has continuous partial derivatives of order ${0,...,C + 1}$ at any point, so we call apply Taylor expansion at $(x,y)$. See full proof at App \ssym~A.2.
Then we set the graphon basis, denoted as $\B_k:[0,1]^2 \rightarrow [0,1]$, to keep the same size as the graphon $\hat{W}$, i.e., $\B_k \in [0,1]^{M \times M}$.
According to the above theorem, instead of direct graphon reconstruction, we can estimate the bases  $\boldsymbol{\B}$ and coefficients $\boldsymbol{\alpha}$. 
Each coefficient $\alpha_k$ reflects how close the corresponding graphon basis $\B_k$ is to the graphon being approximated. 
To summarize, we devise our graphon reconstruction module $\Omega$ as another GNN to transform the encoded node representation $H$ and graph structure $A$ into coefficients 
$\boldsymbol{\alpha}=\{\alpha_1,...,\alpha_C\}$: 
\begin{equation} \label{eq:Wencoder}
    \boldsymbol{\alpha}=\Psi(A,H), 
\end{equation} 
\noindent where $\Psi:(A,H) \rightarrow \mathbb{R}_{+}^C$ maps the embeddings ${H}$ from Eq~\ref{eq:standard_gnn} into $C$ coefficients of graphon bases with $\sum \alpha_i=1$.
\model also learns a set of graphon bases $\B_i \in \mathbb{R}^{K \times K},  \B_i \in \boldsymbol{\B}$ to get the optimal bases to reconstruct the oracle graphon $W$.
Overall, the amount of the parameters can be greatly reduced from $\mathcal{O}(|V|hM^2)$ to $\mathcal{O}(C \cdot (M^2 +|V|d))$ since $C \ll \min\{M^2, |V|d\}$. 
It is worth noting that bases $\boldsymbol{\B}$ are graphons related with the structure $\{A_n\}_{n=1}^N$ of downstream graphs. Therefore, we initialize each basis $\B_k$ as follows: (1) for each run we randomly select a downstream graph $G_i \in \mathcal{G}_t$ and its adjacency matrix $A_{i}$; (2) sort the adjacency matrix in descending order of its node degrees; 
(3) randomly draw a graphon with the size $M$ from the sorted $A_{i}$ 
as the initialized $b_k$.
Moreover, we constrain the learbable bases between $[0,1]$ by setting ${\B}_k = \sigma(b_{k})$ where $\sigma$ is Sigmoid function. 
Finally, as the Theorem~\ref{theo:taylor} implies, the approximated graphon $\hat{W}$ can be aggregated as,
\begin{equation}
\hat{W} = \sum^{C}_{k=0} \boldsymbol{\alpha}_k \boldsymbol{\B}_k.
\end{equation}

Moreover, as the theory indicates, the number of bases corresponds to the order of Taylor expansion (\ie, the more accurate the approximation). 
In \model, we have two major hyper-parameters: the number of learnable bases $C$ and graphon size $M$. The overall learning process of \model can be found in Algorithm 1 from App \ssym~A.3.

\vpara{Complexity Analysis.} 
We now analyze the additional time complexity of \model besides vanilla tuning. Suppose $|V|$ and $|E|$ are the average number of nodes and edges, $d$ is the hidden dimension and $C$ is number of graphon bases.
The total time complexity of \model includes two parts:
(i) the graphon decoder costs $\mathcal{O}(CM^2 + |V|d)$ ;
(ii) the oracle graphon estimation costs $\mathcal{O}(|E|D+|V|D^2)$~\cite{xu2021learning}, where $D$ is the size of oracle graphon.
Thus, the overall additional time complexity is $\mathcal{O}(|E|D+|V|D^2+CM^2+|V|d)$, which is the same magnitudes with the vanilla tuning process of $\mathcal{O}(|E|d+|V|d)$ assuming $M,D \ll |V|$.

\subsection{Theoretical analysis} \label{subsec:theory}
In this section, we further illustrate the ability of \model to capture the discriminative subgraphs present in the downstream graph $\mathcal{G}_t$.
We first give the definition of the discriminative subgraph.

\begin{definition}[Discriminative subgraph] \label{def:discriminative}
A discriminative subgraph $F_{G}$ of graph ${G}$ with a label ${Y}$ is the subgraph $F\subseteq {{G}}$  that minimize the information from ${G}$ and maximize the information to the label ${Y}$ formulated as the optimization:
\begin{equation}
    F_{{G}} = \mathop{\arg\max}\limits_{F\subseteq {G}}{\left[ I({Y}; F)-\beta I({G}; F) \right]}, 
\end{equation}
\noindent where $I(\cdot;\cdot)$ denotes the mutual information and the positive number $\beta$ operate as a tradeoff parameter.
\end{definition}
According to the definition, discriminative subgraphs are the minimal subsets of subgraphs that determine the label of a graph. For example, in the case of a molecular graph, the benzene ring is a discriminative subgraph that distinguishes benzene. Therefore, downstream tasks can benefit from the ability to preserve such discriminative subgraphs.
We provide a theoretical insight below that \model is capable of preserving discriminative subgraphs while reconstructing graphons of downstream graphs.

\begin{theorem}\label{theo:discriminative}
Given an arbitrary graph ${G}$, the oracle graphon ${W}_{{G}}$, the predicted graphon ${\hat{W}}_{G}$, and a discriminative subgraph ${F}_{G}$, the upper bound of the difference between the homomorphism density of $F_{G}$ in the oracle graphon ${W}_{G}$ and that of the predicted graphon ${\hat{W}}_{{G}}$ is decribed as
\begin{equation}
        |t(F_{G}, \hat{W}_{{G}} ) - t(F_{G}, W_{{G}}) | \leq  \frac{\mathrm{e}(F_{G})}{C}  ||R_{C+1}||_{\infty}, 
\end{equation}
\noindent where $\mathrm{e}(F)$ is the number of nodes in subgraph $F$ and $R_{C+1}$ is the remainder in Theorem~\ref{theo:taylor}. 
\end{theorem}
See detailed proof at App \ssym~A.2.
Assuming oracle graphon captures the discriminative subgraphs at $t(F_{G}, W_{{G}})$, \model is optimized to preserve these discriminative subgraphs during fine-tuning.
Moreover, we discuss the generalization bound of \model in App \ssym~A.2.

%% file: 5-exp.tex
\section{Experiments}\label{sec:exp}

\begin{table*}[ht!] 
	\centering
    \setlength{\abovecaptionskip}{-0.05cm} 
    \setlength{\belowcaptionskip}{-0.05cm} 
  	\caption{Mean and standard deviation of ROC-AUC(\%) of fine-tuning on the same domain with pre-training dataset.}
\resizebox{2.0\columnwidth}{!}{
	\begin{tabular}
{ c | c  c  c  c  c  c  c  c |c }
\toprule
 &BBBP & Tox21 & Toxcast & SIDER & ClinTox & MUV & HIV & BACE &Avg.Rank \\ 		\hline

Supervised  &   65.84$\pm$4.51   &   74.02$\pm$0.89  &    61.45$\pm$0.64     &   57.29$\pm$1.65   &   58.06$\pm$4.42     &   71.81$\pm$2.56   &   75.31$\pm$1.93 &   70.13$\pm$5.41 &   10.1 \\ \midrule
Vanilla Tuning & 68.99$\pm$3.13  &  75.39$\pm$0.94 & 63.33$\pm$0.74  & 59.71$\pm$1.08  & 65.92$\pm$3.78  & 75.78$\pm$1.73  & \underline{78.21$\pm$0.51} & 79.22$\pm$1.20 & 5.8 \\ \midrule

 StochNorm & {69.27$\pm$1.67} & 74.97$\pm$0.77 & 62.69$\pm$0.69 & 60.35$\pm$0.98 & 65.53$\pm$4.25 & 76.05$\pm$1.61 & {77.58$\pm$0.84} & 81.48$\pm$2.10 & 6.8\\

 Feature Map & 60.42$\pm$0.78 & 70.58$\pm$0.28 & 61.50$\pm$0.24 & \underline{61.66$\pm$0.45} & 64.05$\pm$3.40 & {78.36$\pm$1.11} & 74.50$\pm$0.49 & 76.32$\pm$1.15 & 8.2  \\ 
 
 L2\_SP & 67.70$\pm$3.21 & 73.55$\pm$0.82 & 62.43$\pm$0.34 & 61.08$\pm$0.73 & 68.12$\pm$3.77 & 76.73$\pm$0.92 & {75.74$\pm$1.50} & 82.21$\pm$2.44 & 7.0 \\
 
 DELTA & 67.79$\pm$0.76 & 75.22$\pm$0.54  & 63.34$\pm$0.58 & 61.43$\pm$0.68 & \underline{72.11$\pm$3.06} & \textbf{80.18$\pm$1.13} & {77.49$\pm$0.88} & 81.83$\pm$1.17 & 4.1\\

 BSS & 68.02$\pm$2.76 & {75.01$\pm$0.71} &  {63.11$\pm$0.49} & 59.60$\pm$0.92 & {70.75$\pm$4.98} & 77.92$\pm$2.01 & {77.63$\pm$0.83} & \underline{82.48$\pm$2.10}  & 5.1\\
 GTOT-Tuning &70.04$\pm$1.48  & 75.20$\pm$0.94  & 62.89$\pm$0.46 &  \textbf{63.04$\pm$0.50} & 71.77$\pm$5.38  & \underline{79.82$\pm$1.78} & 78.13$\pm$1.13 & 82.48$\pm$2.18 &\underline{3.2}  \\ 
 
 VGAE-Tuning & \underline{71.69$\pm$0.51} & \underline{75.79$\pm$0.41} & \underline{63.93$\pm$0.38} & 61.33$\pm$0.28 & 68.63$\pm$1.30 & 77.70±1.43 & 77.50$\pm$0.13 & 77.83$\pm$0.64 & 4.6 \\ \midrule
 
\model w/o Pre  &   65.91$\pm$2.54 &   70.09$\pm$0.73  &    59.79$\pm$0.62     &   58.31$\pm$2.67   &   61.04$\pm$2.55      &   73.39$\pm$1.88 &   75.55$\pm$2.06   & 81.63$\pm$0.71 & 9.5 \\

 \model& \textbf{72.59$\pm$0.32} &\textbf{75.80$\pm$0.29}   & \textbf{64.25$\pm$0.27}  & {61.40$\pm$0.71}   & \textbf{74.64$\pm$4.30}  & {75.84$\pm$1.97}   & \textbf{78.33$\pm$0.67}  & \textbf{84.79$\pm$1.39}&\textbf{1.5} \\
\bottomrule
	\end{tabular}
}
\vspace{-0.5cm}
	\label{tab:sota1}
\end{table*}
In this section, we aim to practically evaluate the performance of \model on 15 datasets  under two settings. 
To be specific, we answer the following questions:

\begin{compactitem}
    \item [\textbf{Q1. (Effectiveness)}] Does \model improve the performance of fine-tuning? 
    \item [\textbf{Q2. (Transferability)}] Can \model enable the better transferability than baselines? 
    \item [\textbf{Q3. (Integrity)}] How does each component of \model contribute to the performance? 
    \item [\textbf{Q4. (Efficiency)}] Can \model improve the performance of fine-tuning at an acceptable time consumption?
\end{compactitem}

\vpara{Baselines.} 
There is a large number of GNN pre-training methods, but only a few fine-tuning strategies are available.
Therefore, we implement several representative baselines in computer vision that were originally designed for CNNs, including StochNorm~\cite{kou2020stochastic}, DELTA and the version with fixed attention coefficients (Feature-Map)~\cite{li2018delta}, L2\_SP\cite{xuhong2018explicit} and BSS\cite{chen2019catastrophic}.
To the best of our knowledge, there is only one baseline dedicated to improving the fine-tuning of GNNs, which is agnostic to the pre-training strategy, namely GTOT-Tuning~\cite{ijcai2022p518}.
To validate the effectiveness of reconstructing graphons, we introduce VGAE-Tuning for comparison, which employs VGAE~\cite{kipf2016variational} as the auxiliary loss to reconstruct the adjacency matrices of downstream graphs.
We reproduce the baselines based on the code released by the authors and set the hyperparameters according to their released code and settings depicted in their papers. More details can be found in App \ssym A.4.

\subsection{Fine-tuning GNNs}\label{sec:Q1}
 
\vpara{Setting.} To answer \textbf{Q1}, we evaluate \model on the molecular property prediction task to show its effectiveness. Following the setting of~\cite{hu2020strategies,ijcai2022p518}, we use the model pre-trained by unsupervised context prediction task as the backbone model. Specifically, we pre-train GIN~\cite{xu2018how} by self-supervised Context Prediction task on the ZINC15 dataset with 2 million unlabeled molecules~\cite{sterling2015zinc}. 
Next, we perform fine-tuning of the backbone model on 8 binary classification datasets obtained from MoleculeNet~\cite{wu2018moleculenet}. We use the scaffold split at an 8:1:1 ratio. 
Since our framework is agnostic to the backbone GNNs, we focus on evaluating whether our model achieves better fine-tuning results. For each dataset, we run 5 times and report the average ROC-AUC with the corresponding standard deviation.

\begin{table*}[h!] 
	\centering
    \setlength{\abovecaptionskip}{-0.01cm} 
  	\caption{Mean and standard deviation of Accuracy(\%) of fine-tuning on different domains from pre-training datasets.}
\resizebox{0.8\textwidth}{!}
{
	\begin{tabular}
{   c | c  c  c  c  c  c c |c}
\toprule
 &IMDB-M & IMDB-B & MUTAG & PROTEINS & ENZYMES & MSRC\_21 
 & RDT-M12K &Avg.Rank  \\ 		\hline

Supervised   &   36.67$\pm$6.67   &   52.40$\pm$7.20   &   82.89$\pm$6.16   &   63.51$\pm$3.60   &   20.50$\pm$4.02   &   7.45$\pm$3.52  &  38.09$\pm$0.56 & 10.1 \\ \midrule
Vanilla Tuning & 50.20$\pm$2.72  &  72.10$\pm$3.65 & 83.45$\pm$5.71  & 64.42$\pm$4.91  & 21.33$\pm$5.62  & 7.99$\pm$ 2.39  & 40.53$\pm$1.21 & 6.6 \\ \midrule

StochNorm & 49.87$\pm$3.11  & 72.20$\pm$2.99  &  82.40$\pm$7.95 & 64.08$\pm$3.61  & 23.33$\pm$4.25 & 9.08$\pm$3.21  & 41.02$\pm$0.98 & 6.0    \\   
 
 Feature Map & 50.87$\pm$2.89  & 72.90$\pm$2.70  & 82.95$\pm$7.80  & 63.06$\pm$4.80  & 22.67$\pm$4.78  & 9.77$\pm$4.30   & 40.74$\pm$1.26 & 5.6 \\   
 
 L2\_SP & 51.07$\pm$2.11  & 71.90$\pm$2.59  & 82.98$\pm$3.91  & 65.95$\pm$4.57  & 21.50$\pm$4.50  & 10.29$\pm$3.91   & 38.36$\pm$0.88  & 6.0   \\
 
 DELTA &  50.67$\pm$2.81 & 71.80$\pm$3.99  & 82.98$\pm$6.12  & 63.96$\pm$5.35  & 22.33$\pm$5.01  & \underline{10.48$\pm$3.84}  & 39.87$\pm$1.47 & 6.6  \\

 BSS & 47.35$\pm$1.76  & \underline{73.20$\pm$3.25}  & 84.56$\pm$5.52   & 65.58$\pm$6.79  & \underline{23.41$\pm$4.79}  & 10.47$\pm$3.29   & 39.90$\pm$1.39 &\underline{ 4.7 } \\
 
 GTOT-Tuning & \underline{51.13$\pm$2.72}  & 72.30$\pm$2.93  & \textbf{87.50$\pm$6.94} &  62.89$\pm$3.88 & 20.67$\pm$5.49  & 8.32$\pm$3.92 & 39.90$\pm$0.85 & 6.0    \\ 
 
 VGAE-Tuning & 49.27$\pm$2.37 & 72.50$\pm$2.91 & 80.35$\pm$6.16 & {67.34$\pm$3.52} & 19.33$\pm$7.61 & 8.88$\pm$2.75 & \underline{41.43$\pm$1.88}  & 6.6 \\ \midrule
 \model w/o Pre &   49.27$\pm$3.09   &   72.10$\pm$4.99   &   83.04$\pm$3.09   &   \underline{69.88$\pm$3.30}   &   20.17$\pm$4.37   &   7.63$\pm$2.95 & 40.92$\pm$1.78  & 6.7 \\
\model & \textbf{51.80$\pm$2.31} & \textbf{74.30$\pm$3.29}  &  \underline{86.14$\pm$5.50}  &  \textbf{72.05$\pm$3.80}  & \textbf{26.70$\pm$4.28}  & \textbf{11.01$\pm$2.08} & \textbf{42.80$\pm$1.62} & \textbf{1.1} \\
\bottomrule
	\end{tabular}
	}
\vspace{-0.3cm}
	\label{tab:sota2}
\end{table*}

\vpara{Results.} 
Tab~\ref{tab:sota1} shows that \model achieves 6 best performance among 8 datasets against the baselines, taking a top average rank. We notice that constraining the embedding from pre-trained model like Feature-Map or DELTA sometimes bring worse performance than vanilla tuning. 
From the comparison between supervised learning and \model w/o Pre, although there might be occasional instances of slight performance drops when applying the \model loss to supervised learning, the majority of supervised training experiences benefits from \model loss.
From the comparison between vanilla tuning and the last two rows of the table, the performance of \model without pre-training is lower than that with pre-training but sometime better than the vanilla tuning.
The results generally prove that in the case that datasets come from the same domains, \model can compensate for structural divergence by preserving generative patterns and lead to better performance.

\subsection{Fine-tuning GNNs across domains}\label{sec:Q2}

\vpara{Settings.} To answer \textbf{Q2}, we propose to evaluate \model in the cross-domain setting where the pre-training and downstream datasets are not from the same domains. It is a more challenging yet more realistic setting~\cite{qiu2020gcc,you2020graph} due to the larger structural divergence which can in turn degrade the performance.  
Therefore, we adopt GCC~\cite{qiu2020gcc} as the backbone model and its subgraph discrimination as the pre-train task.
Following the setting of GCC, we pre-train on 7 different datasets ranging from academia to social domains, and evaluate our approach on 7 downstream graph classification benchmarks: IMDB-M, IMDB-B, MUTAG, PROTEINS, ENZYMES, MSRC\_21 and RDT-M12K from the TUDataset~\cite{Morris+2020}. These datasets cover a wide range of domains including movies, chemistry, bioinformatics, vision graphs and social network.
We report the results under 10-fold cross-validation.

\vpara{Results.} From Tab~\ref{tab:sota2}, we find that our model outperforms all baselines on 6 out of 7 datasets and presents competitive results on MUTAG (1.92\% lower than the best). \model improves performance on PROTEINS by 7.63\% and 4.71\% when compared to vanilla tuning and the second best baseline respectively.
We can observe a more substantial improvement of \model compared with the previous experiment (Tab~\ref{tab:sota1}), because we explicitly preserve the generative patterns.
Although GTOT also incorporates structural information, it sometimes even performs worse than vanilla tuning (\ie PROTEINS and ENZYMES).
Generally, \model clearly demonstrates its effectiveness when pre-training and fine-tuning graphs exhibit large structural divergence.

\begin{figure}[h] 
    \centering
    \includegraphics[width=\columnwidth]{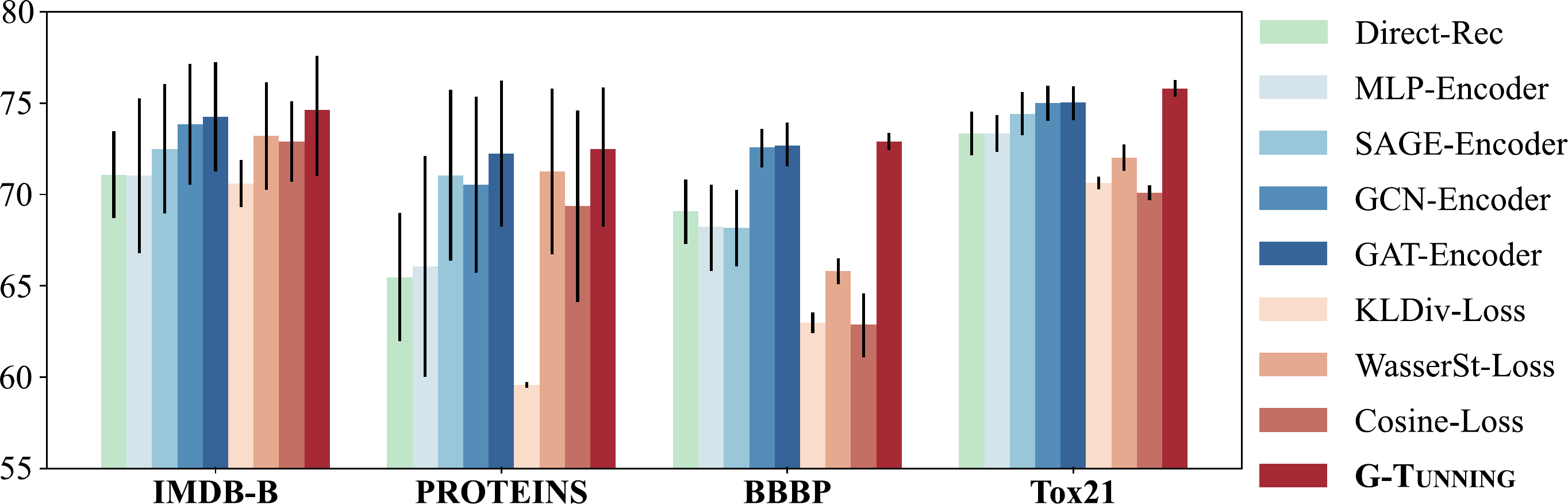}
    \setlength{\abovecaptionskip}{-0.2cm}  
    \setlength{\belowcaptionskip}{-0.5cm} 
    \caption{Ablation study on \model varying architecture and training objective.}
    \label{fig:ablation}
\end{figure}

\subsection{Model Ablation and Hyper-parameter Study}

\begin{figure*}[ht!]
\begin{minipage}{0.30\linewidth}
\centering
\includegraphics[width=1.0\textwidth]{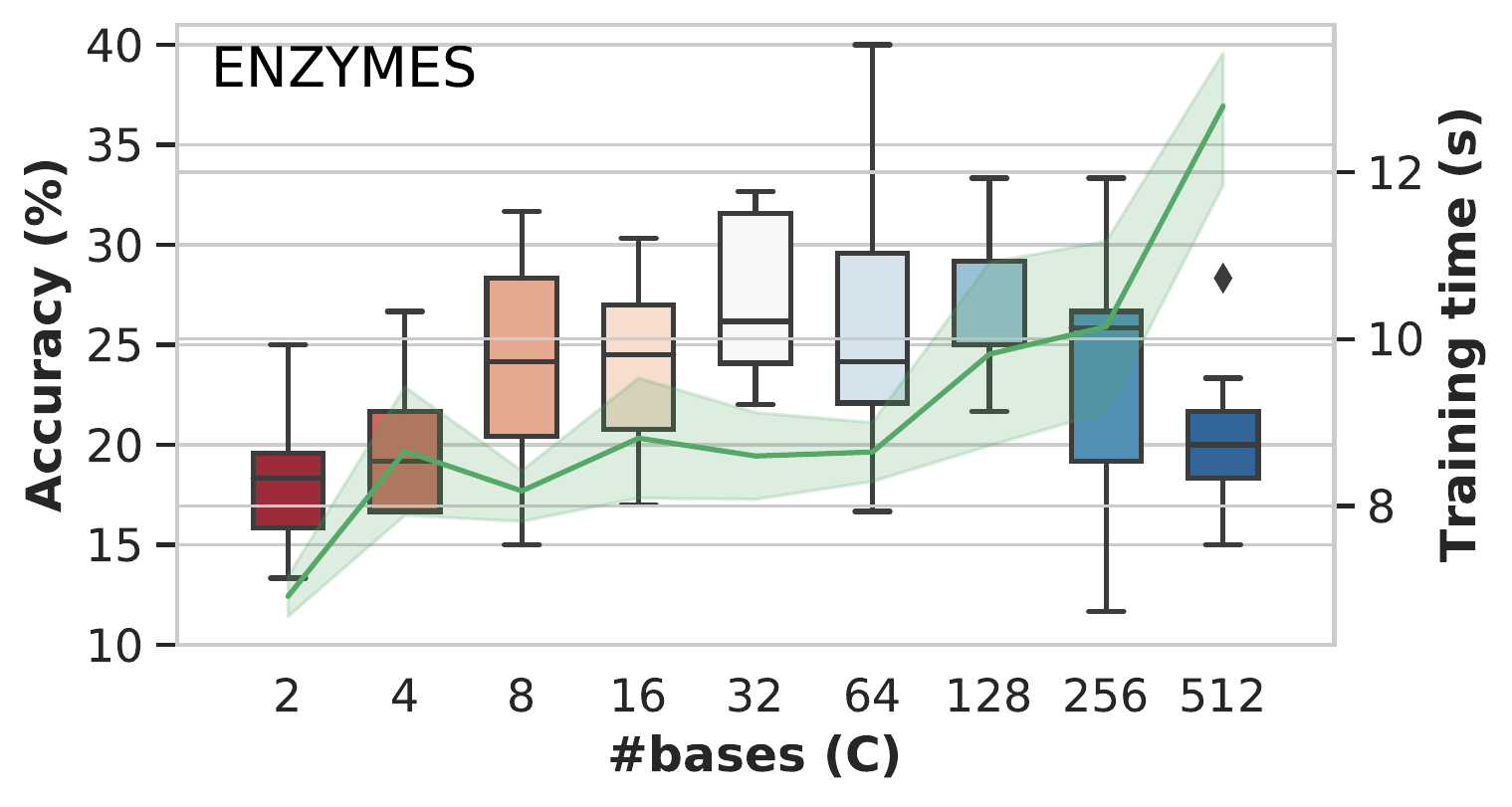}  
\setlength{\abovecaptionskip}{-0.4cm} 
\setlength{\belowcaptionskip}{-0.5cm} 
\captionof{figure}{Performance \& time varying \# graphon bases.}         
\label{fig:sensitivity_b}
\end{minipage}
\hfill
\begin{minipage}{0.69\linewidth}
    \begin{table}[H]
    \centering
     \label{tab:RunningTime}
    \setlength{\abovecaptionskip}{-0.03cm} 
    \caption{The Comparison of Running time.}
\resizebox{1.0\textwidth}{!}
{
	\begin{tabular}
{ c |c c c c c  c  c  c }
\toprule

       &   BBBP    &   Tox21    &   Toxcast   &   SIDER   &  ClinTox  &     MUV      &    HIV     &   BACE    \\ 		\midrule

      \# graphs      &   2039    &    7831    &    8575     &   1427    &   1478    &    93087     &   41127    &   1513    \\ \hline
  Vanilla Tuning   & 1.95$\pm$0.12 & 5.02$\pm$0.26  & 88.96$\pm$3.09  & 2.33$\pm$0.19 & 1.87$\pm$0.07 &  81.56$\pm$3.21  & 15.62$\pm$0.47 & 1.82$\pm$0.08 \\
    L2\_SP        & 2.22$\pm$0.08 & 8.93$\pm$0.82  & 95.62$\pm$4.93  & 2.66$\pm$0.22 & 1.83$\pm$0.13 & 118.64$\pm$11.41 & 86.06$\pm$4.23 & 1.86$\pm$0.04 \\
      DELTA        & 2.04$\pm$0.10 & 6.26$\pm$0.14  & 91.10$\pm$2.93  & 2.39$\pm$0.17 & 1.72$\pm$0.80 & 109.69$\pm$14.22 & 99.13$\pm$0.21 & 1.01$\pm$0.39 \\
   Feature Map     & 1.91$\pm$0.16 & 6.62$\pm$0.27  & 96.35$\pm$5.33  & 3.82$\pm$0.49 & 1.80$\pm$0.23 & 113.62$\pm$11.09 & 31.69$\pm$5.20 & 1.99$\pm$0.20 \\
       BSS         & 2.09$\pm$0.20 & 9.30$\pm$0.41  & 94.07$\pm$2.82  & 2.56$\pm$0.17 & 2.11$\pm$0.17 & 114.84$\pm$12.50 & 72.04$\pm$5.81 & 2.12$\pm$0.06 \\
    StochNorm      & 2.05$\pm$0.13 & 7.82$\pm$0.48  & 96.11$\pm$7.14  & 2.36$\pm$0.19 & 1.70$\pm$0.13 & 120.37$\pm$13.24 & 67.26$\pm$5.83 & 2.05$\pm$0.14 \\
       GTOT        & 1.97$\pm$0.11 & 5.71$\pm$0.18  & 95.87$\pm$3.01  & 2.42$\pm$0.15 & 1.92$\pm$0.08 &  94.62$\pm$3.90  & 59.32$\pm$0.97 & 1.93$\pm$0.12 \\
    VGAE-Tuning     & 5.67$\pm$0.07 & 17.47$\pm$0.30 & 108.86$\pm$0.28 & 5.62$\pm$0.17 & 4.71$\pm$0.10 & 238.98$\pm$15.78 & 95.59$\pm$3.68 & 5.69$\pm$0.09 \\
     \model      & 4.36$\pm$0.19 & 15.53$\pm$0.22 & 102.32$\pm$1.43 & 5.44$\pm$0.23 & 4.06$\pm$0.15 & 198.04$\pm$15.82 & 73.81$\pm$1.35 & 4.96$\pm$0.19 \\ \hline
 Oracle W Est &  55.70\%   &   50.97\%   &   51.68\%    &  72.50\%   &  68.01\%   &    56.65\%    &   63.88\%   &  69.72\%   \\
\bottomrule
	\end{tabular}
	}
    \end{table}
\end{minipage}
\vspace{-0.3cm}
\end{figure*}

\vpara{Model ablation study.} To answer \textbf{Q3}, we choose 2 datasets from above 2 experiments to perform ablation study (Fig~\ref{fig:ablation}). 
To begin with, we examine the effectiveness of proposed graphon decomposition by comparing with reconstructing graphon directly (``Direct-Rec''). 
The results show that the improvement beyond vanilla-tuning is limited, and in some cases, there is even negative transfer. The reason may lie in the difficulty of directly reconstructing the complex semantic information and capturing the properties of graphon from the $A$ and $H$.
Across four datasets, we observe \model always outperforms ``Direct-Rec''.
Next, we compare different GNN architectures (two-layered MLP, GCN\cite{welling2016semi}, GraphSAGE\cite{hamilton2017inductive} and GAT~\cite{2017Graph}) with the default backbone (\ie GIN~\cite{xu2018how}).
In Fig~\ref{fig:ablation}, we observe MLP-encoder performs the worst, which proves the effectiveness of incorporating structural information to reconstruct graphon.  
Lastly, we replace our loss with KL divergence, Wasserstein distance and cosine similarity. 
We can observe that our GW discrepancy loss significantly outperforms the others. 
We argue that cosine similarity may be insensitive to the absolute value which refers to the probability of edges. 
Since KL divergence does not satisfy the commutative law, it is difficult to converge when reconstructing graphons.
Though Wasserstein distance is also based on optimal transport, it fails to capture the geometry between two graphons.

\vpara{Hyper-parameter Study.} Moreover, we study the effect of different graphon bases from 2 to 512.
As analyzed in Theorem~\ref{theo:taylor},  
more bases can represent more information and better approximate the oracle graphon. 
Fig~\ref{fig:sensitivity_b} shows that the performance increases as the number of bases grows from 2 to 32. 
However, when the number keep increasing, the improvement becomes smaller.
We attribute this phenomenon to the optimization difficulty brought by the increased number of parameters.
Besides, as the number of bases increases, the running time of \model grows exponentially (green curves). 
Therefore, \model only requires a small amount of bases to improve the fine-tuning performance.

\subsection{Running Time}
To answer \textbf{Q4}, we conduct a running time comparison (Tab 3).
Firstly, we conducted a complexity analysis in Sec~\ref{sec:Optimization} and presented the running time comparison between our method and baselines in Tab 3. The time complexity mainly consists of two parts: (i) the graphon approximation from the pre-trained model and (ii) the oracle graphon estimation. Secondly, we now report the computing efficiency of our method and baselines on the datasets under the in-domain setting without loss of generality (seconds per training epoch). Please note that \model incurs additional time only during the training process of the fine-tuning stage, while the rest of the time consumption is the same as the vanilla tuning. From Tab 3, we can see that \model is not the slowest tuning method in most cases. It can be observed that our method's time is close to VGAE-tuning. As indicated by the number of graphs in the first row, it is evident that as the number of graphs increases, the time consumption of our method becomes more comparable to that of other baselines. It means our method exhibits excellent scalability. Thus, the time consumption of our method is kept in an acceptable range. Moreover, the last row shows the mean proportion of the oracle graphon estimation during the complete training process in the fine-tuning stage. It is evident that a significant portion of time is dedicated to the oracle graphon estimation. If we exclude this portion of time, our overall time consumption is almost the same as that of the vanilla tuning. If the graphon estimation research advances and can run faster, our method will also benefit from corresponding speed improvements.
Due to the space limit, please refer to other experiments in App \ssym~A.5.

%% file: 6-conclusion.tex
\section{Conclusion}

In this paper, we attribute the unsatisfactory performance of the pre-training on graphs to the structural divergence between pre-training and downstream datasets.
Moreover, we identify the cause of this divergence as the discrepancy of generative patterns between pre-training and downstream graphs. Building upon our theoretical analysis, we propose \model, a GNN fine-tuning strategy based on graphon, to adapt the pre-trained model to the downstream dataset.
Finally, we empirically prove the effectiveness of \model.

%% file: 7-appendix.tex
\appendix
\section{Appendix}

\subsection{Related Work}
\input{6.5-related}

\clearpage

\subsection{Theorems and Proofs}  \label{app:theroems}
\subsubsection{Proof of Theorem~\ref{theo:taylor}} 

In Section~\ref{sec:model}, we provide a fine-tuning framework aiming at adapting the pre-trained GNN to the downstream graphs by preserving the generative patterns of these graphs.

During fine-tuning, a pre-trained GNN $\Phi$ obtains latent representations $H$ for downstream graphs $\mathcal{G}_t = \{G_1,...,G_n\}$ and feeds them into task-specific layers $f_\phi$ to train with the fine-tuning labels $Y$. 
However, the above vanilla strategy may fail to achieve satisfactory performance due to the large discrepancy between pre-training and fine-tuning graphs, and in some instances, it may inadvertently result in \emph{negative transfer}. 
To alleviate this phenomenon, we propose to enable the pre-trained GNN $\Phi$ to preserve the generative patterns of the downstream graphs $\mathcal{G}_t$ by reconstructing their graphons $W$.

To this end, we propose that any graphon $W(x, y)$ can be reconstructed by a linear combination of graphon bases $\B_k \in \boldsymbol{\B}$. 

\vpara{Theorem 1}
$\forall \ W(x, y)\in\mathcal{W}_{C + 1}$,  there exists $C$ graphon bases ${\B_k(x, y)}$ that satisfies $W(x, y) = \sum^{C}_{k = 1}\alpha_k \B_k(x, y) + R_{C+1}(x, y)$, where $\alpha_i\in \mathbb{R} $ and $R_{C+1}(x, y)$ is the remainder of order $C+1$.


\begin{proof}
For an arbitrary graphon $W$, we have its Taylor expansion at the point (0, 0): 

\begin{eqnarray}
\begin{aligned}
W(x,y) =& \;W(0,0) + (x\frac{\partial }{\partial x}+y\frac{\partial }{\partial y})W(0,0) \\
 & + \frac{1}{2!} (x\frac{\partial }{\partial x}+y\frac{\partial }{\partial y})^2W(0,0)+... \\
 & + \frac{1}{C!} (x\frac{\partial }{\partial x}+y\frac{\partial }{\partial y})^C W(0,0) + R_{C+1} \\
 =&\sum^{C}_{k=0} \frac{1}{k!} \sum^{k}_{i=0}  C_k^i \frac{\partial ^k W(0,0) }{\partial x^i \partial y^{k-i}} x^i y^{k-i} +R_{C+1}\\
 =&\sum^{C}_{k=0} \frac{I_k}{k!} \left( \sum^{k}_{i=0}  \frac{C_k^i}{I_k} \frac{\partial ^k W(0,0) }{\partial x^i \partial y^{k-i}} x^i y^{k-i} \right) +R_{C+1}\\
 =&\sum^{C}_{k=0} A_k \B_k(x,y) +R_{C+1},
\end{aligned}
\end{eqnarray}
\noindent where $I_k= \sum^{k}_{i=0}  C_k^i \frac{\partial ^k W(0,0) }{\partial x^i \partial y^{k-i}}, A_k=\frac{I_k}{k!}, \B_k(x,y)=\sum^{k}_{i=0} \frac{C_k^i}{I_k} \frac{\partial ^k W(0,0) }{\partial x^i \partial y^{k-i}} x^i y^{k-i}.$

Since $C_k^i \frac{\partial ^k W(0,0) }{\partial x^i \partial y^{k-i}} \leq I_k$, the coefficient in $\B_k(x,y)$ is in $[0,1]$, the $x^i y^{k-i}$ is in $[0,1]$.
Thus, $\B_k(x,y)$ is a bounded symmetric function from $[0,1]^2$ into $[0,1]$, which preserve the properties of graphon, i.e. $\B_i(x,y)$ can be a graphon and is called a graphon basis. 
Moreover, we can approximate the predicted graphon $\hat{W}(x,y)$ by $\sum^{C}_{k=0} A_k \B_k(x,y)$, which indicating the linear combination of the graphon bases $\B_k(x,y)$.
\end{proof}

\subsubsection{Proof of Theorem~\ref{theo:discriminative}} \label{subapp:discrim}
Here we prove that \model preserves not only the generative patterns of downstream graphs but also discriminative subgraphs (Definition~\ref{def:discriminative}) which the downstream task can benefit from.

To prove the Theorem ~\ref{theo:discriminative}, we first introduce two lemmas. First, we introduce the counting lemma for graphons:
\begin{lemma}[Counting Lemma for Graphons \cite{lovasz2012large}]\label{lemma:Counting Lemma}
Let $F$ be a simple subgraph and $W, W^{\prime} \in \mathcal{W}$ be two graphons. 
Then we have:
\begin{equation}
\begin{aligned}
     |t(F,W) - t(F,W^{\prime}) | \leq \mathrm{e}(F)|| W-W^{\prime} ||_{\square},
\end{aligned}
\end{equation}
\noindent where homomorphism density $t(F,W)$ which mentioned in Preliminaries (Section 3) is used to measure the relative frequency that homomorphism of graph $F$ appears in graph $G$, $\mathrm{e}(F)$ is the number of nodes in graph $F$, and $|| \cdot ||_{\square}$ denotes the cut norm of graphon in~\cite{lovasz2012large}.
\end{lemma}

Next, we introduce the lemma of cut norm of a graphon:
\begin{lemma}\label{lemma:cut norm}
The cut norm  of a graphon $\|W\|_{\square}$ is defined as 
\begin{equation}
\|W\|_{\square}=\sup _{S, T \subseteq[0,1]}\left|\int_{S \times T} W\right|,
\end{equation}
\noindent where the supremum is taken over all measurable subsets $S$ and $T$.
Obviously, suppose $\gamma \in \mathbb{R}$, we have
\begin{equation}
\|\gamma W\|_{\square}=\sup _{S, T \subseteq[0,1]}\left|\int_{S \times T} \gamma W\right|=\sup _{S, T \subseteq[0,1]}\left|\gamma \int_{S \times T} W\right|=\gamma\|W\|_{\square}.
\end{equation}

\end{lemma}

We now present the theorem that asserts the capability of \model to retain the discriminative subgraph in downstream graphs.

\vpara{Theorem 2}
Given an arbitrary graph $G$, its oracle graphon ${W}$, its predicted graphon ${\hat{W}}$, and a discriminative subgraph ${F}$ in $G$, the difference between the homomorphism density of $F$ in the oracle graphon ${W}$ and that in the predicted graphon ${\hat{W}}$ is upper bounded by 
\begin{equation}
\begin{aligned}
        |t(F, \hat{W} ) - t(F, W) | &\leq  \frac{\mathrm{e}(F)}{C}  ||R_{C+1}||_{\infty}, 
\end{aligned}
\end{equation}


\noindent where $R_{C+1}$ is the remainder in Theorem~4.1 in the main paper. 


\begin{proof}
Applying Lemma~\ref{lemma:Counting Lemma}, we only need to prove the following inequality
\begin{equation}
\begin{aligned}
        \mathrm{e}(F)|| W-W^{\prime} ||_{\square} &\leq  \frac{\mathrm{e}(F)}{C}  ||R_{C+1}||_{\infty}.
\end{aligned}
\end{equation}
Moreover, according to the Theorem~4.1 in the main paper, we have $\mathrm{e}(F)||W-W^{\prime}||_{\square}=\mathrm{e}(F)||R_{C+1}||_{\square}$. 

Then, applying Lemma~\ref{lemma:cut norm}, the difference between the homomorphism density of $F$ in the oracle graphon ${W}$ and that in the predicted graphon ${\hat{W}}$ can be written as 


\begin{eqnarray}
\begin{aligned}
|t(F, \hat{W} ) - t(F, W) | &\leq \mathrm{e}(F)|| W-W^{\prime} ||_{\square} \\
    &=\mathrm{e}(F)||R_{C+1}||_{\square} \\
    &=\mathrm{e}(F) \Bigl| \int_{[0,1]\times[0,1]}R_{C+1} \rm{d}x\rm{d}y \Bigr|\\
    &=\mathrm{e}(F) \Bigl| \int_{[0,1]\times[0,1]} \frac{1}{(C+1)!}  (x\frac{\partial }{\partial x}+y\frac{\partial }{\partial y})^{C+1} W(x_0, y_0) \rm{d}x\rm{d}y \Bigr|\\
    &=\mathrm{e}(F) \Bigl| \int_{[0,1]\times[0,1]} \sum_{i=0}^{C+1} \frac{1}{i! (C+1-i)!} \frac{\partial W(x_0,y_0)}{\partial x^i \partial y^{C+1-i}} x^i y^{C+1-i} \rm{d}x\rm{d}y \Bigr|\\
    &=\mathrm{e}(F) \sum_{i=0}^{C+1} \frac{1}{(i+1)! (C+2-i)!} \frac{\partial W(x_0,y_0)}{\partial x^i \partial y^{C+1-i}}\\
    &\leq \frac{\mathrm{e}(F)}{C} \sum_{i=0}^{C+1} \frac{1}{i! (C+1-i)!} \frac{\partial W(x_0,y_0)}{\partial x^i \partial y^{C+1-i}}\\
    &\leq \frac{\mathrm{e}(F)}{C} R_{C+1}(1,1)\\
    &\leq \frac{\mathrm{e}(F)}{C} ||R_{C+1}||_{\infty},
\end{aligned}
\end{eqnarray}
\noindent where $||R_{C+1}||_{\infty}$ is $\sup _{x,y} R_{C+1}(x,y)$.
\end{proof}

\subsubsection{Generalization Bound of \model} \label{subapp:bound}
We analyse the generalization bound of \model and expect to find the key factors that affect its generalizability.
As a preliminary, we first introduce the uniform stability.
\begin{definition}[Uniform Stability]  \label{def:uniform}
Let $ S=\{z \mid z_i = (G_i,y_i), i= 1,...,N \} $ denote a training set with $N$ graphs, and $S^{i}$ denote the training set replacing graph $i$ with graph $i'$. A function $f$ has uniform stability $\beta$ with respect to the loss function $l$ if it satisfies:
\begin{equation}
    \forall S \in \mathcal{Z}, \quad \forall i\in \{ 1,...,N \}, \quad \mid l(f_S,z)-l(f_{S^{i}},z) \mid \leq \beta 
\end{equation}
\end{definition}

Then, we introduce the lemma of generalization bound. 
\begin{lemma}[Generalization Bound\cite{bousquet2002stability}] \label{lemma:gen bound}
Let $f$ denote a learning algorithm with a loss function $l$, and $\beta$ denote its uniform stability. 
For $\forall z \in \mathcal{Z} $ and  $\forall \epsilon \textgreater 0$, with $0 \leq l(f_S, z) \leq U$, the function $f_S$ trained over $S$ with $N$ graphs has following generalization bound:
\begin{equation}
        \mathbb{P}_S \bigl[ R(f_S) -  R_N(f_S) \textgreater \epsilon + 2\beta \bigr] \leq \exp{\bigl( -\frac{2N \epsilon^2}{(4N\beta +U)^2} \bigr) }, 
\end{equation}
\noindent where $\beta$ is the uniform stability in Definition~\ref{def:uniform}, $R(f_S)$ denotes the generalization error $E_z[l(f_S,z)]$ of $f$ on the data space $S$ and $R_N(f_S)$ denotes the empirical error $\frac{1}{N}\sum_{i=1}^N l(f_S,z_i)$.
\end{lemma}

Following the above lemma, we obtain the generalization error bound of \model as follows.
\begin{theorem}[Generalization Bound for \model]
\label{pro:generalization_bound}
Assume that \model satisfies $0\leq l(f_S,z)\leq U$ and cross entropy loss $0\leq \mathcal{L}_\text{CE}(f_S,z)\leq P$. For $\forall \epsilon \geq 0$, the following inequality holds over the random draw of the sample $S$ with $N$ graphs:
\begin{equation}
    \mathbb{P}_S \bigl[ R(f_S) -  R_N(f_S) \textgreater \epsilon + 2P+4\lambda \bigr] \leq \exp{\bigl( -\frac{2N \epsilon^2}{(4N P+8N\lambda +U)^2} \bigr) } 
\end{equation}
\noindent where $\lambda$ denotes the hyper-parameter of \model loss.
\end{theorem}

\begin{proof}
We derive the uniform stability of our \model and substitute the $\beta$ in Lemma~\ref{lemma:gen bound} to prove the above theorem. 

$l(f_S, z)$ is the overall loss, which contains $\mathcal{L}_\text{CE}$ and $\mathcal{L}_\model$ shown in Equation~4 in the main paper. Specifically, we extend $\mathcal{L}_\model$ according to Equation~6 in the main paper.

For any training set $S$ and $S^{\prime}$ as any sample $(G,y)\in \mathcal{Z}$ in $S$ replaced, $\hat{W}^S, \hat{W}^{S^{\prime}}$ respectively denote the predicted graphon predicted using $S, S^{\prime}$, then our uniform stability is formulated as:

\begin{eqnarray}
\begin{aligned}
    &\;\;\;\;|l(f_S,z)-l(f_{S^{\prime}},z)| \\
    &= \bigl| \mathcal{L}_\text{CE}(f_{S}(G),y) - \mathcal{L}_\text{CE}(f_{S^{\prime}}(G),y) + \\
    & \lambda \sum_{i,j,k,l}(W^{S}_{i,k}-\hat{W}^{S}_{j,l})^p T^{S}_{i,j} T^{S}_{k,l}-\lambda \sum_{i,j,k,l}(W^{S^{\prime}}_{i,k}-\hat{W}^{S^{\prime}}_{j,l})^p T^{S^{\prime}}_{i,j} T^{S^{\prime}}_{k,l} \bigr|   \\
    & \leq P +2\lambda  \sum_{i,j,k,l} | (\hat{W}^{S}_{i,k}-{W}_{j,l})^p-(\hat{W}^{S^{\prime}}_{i,k}-{W}_{j,l})^p | T_{i,j}T_{k,l} \\
    & \leq P +2\lambda  \sum_{i,j,k,l}1 \cdot T_{i,j}T_{k,l} = P +2\lambda
\end{aligned}    
\end{eqnarray}
Then we substitute the $\beta = P + 2\lambda$ in Lemma~\ref{lemma:gen bound} and complete this proof.
\end{proof}

\subsubsection{Notations}

Here we collect most of the notations in Section~\ref{sec:model}:

\begin{table}[h]
\centering
\begin{tabular}{|c|c|}
\hline
Symbol & Definition \\
\hline
$G$ & A graph consists of structure $A$ with nodes $V$ and edges $E$, \\
& node attribute matrix $X$ and graph label $Y$. \\
$d$   &  The hidden dimention of node attribute matrix $X$.   \\ 
\hline
${\Phi}$ & The pre-training backbone. \\
\hline
$f_{\phi}$ & The shallow downstream adaption layer.\\
\hline
$H$ & The node embedding matrix from pre-trained model $f_{\Phi}$. \\
\hline
${\Omega}$ & The process of graphon reconstruction. \\
\hline
$\Psi$   & The \Wencoder.         \\ \hline
$\hat{W} \& W$ & The predicted graphon and the oracle graphon respectively. \\
\hline
$\alpha$   &   The coefficients        \\ \hline
$\B$ & The graphon bases. \\ \hline
$D$ & The size of oracle graphon. \\ \hline
$\lambda$ & The tradeoff parameter. \\ \hline
$M$ & The size of predicted graphon. \\ \hline
$C$   &  The number of bases $\B$.         \\ \hline
$T$   & Transportation matrix of our $p$-order GW distance.         \\ \hline
$I(\cdot;\cdot)$   &     The mutual information.      \\ \hline
\end{tabular}
\caption{Notations in Section~\ref{sec:model}}
\label{tab:notation}
\end{table}

\clearpage

\subsection{Algorithm} \label{app:algo}
As described Section~\ref{sec:model}, we introduce a novel fine-tuning strategy. Here we try to provide the algorithms in details.
\subsubsection{Oracle Graphon Estimation} \label{subapp:oracle}
In Section 4.2, we propose that we need the oracle graphon of the downstream graphs to adapt the pre-trained GNN to the generative patterns of the downstream dataset.
Thus, we need to estimate the oracle graphon $W$ by downstream graphs $\mathcal{G}_t$. 

\begin{algorithm}[ht]
\caption{Estimating Oracle Graphon} 
\label{alg1}
\begin{flushleft}
  \textbf{Input:} Adjacency matrix $A$. The weight of proximal term $\beta$, the number of iterations $L$, the number of inner Sinkhorn iterations $S$. The size of Oracle Graphon $D$. \\
  Initialize $W\sim \text{Uniform}([0,1])$. \\
  Initialize $\mu=\frac{A}{\|A\|_1}$ 
 and $\mu_W=\text{interp1d}_{D}(\text{sort}(\mu))$ 
 , where $\text{interp1d}_{D}(\cdot)$ samples $D$ values from the input vector via linear interpolation. \\
 \textbf{Output:} The estimated oracle graphon $W$.
\end{flushleft}
\begin{algorithmic}[1]
     \FOR{$l=1,...,L$:}
       \STATE Initialize $T^{(0)}=\mu \mu_W^{\top}$ and $a=\mu$. \\
      \FOR{$s=0,...,S-1$:} 
      \STATE $\widetilde{D} =(A\odot A)\mu\bm{1}_{D} + \bm{1}_{N}\mu_W^{\top}(W\odot W)$; \\
      $Q=\exp(-\frac{1}{\beta}(\widetilde{D} - 2AT^{(s)}W^{\top}))\odot T^{(s)}$;  \\
      $b=\frac{\mu_{W}}{Q^{\top}a}$, \quad $a = \frac{\mu}{Q b}$, \quad $T^{(s+1)}=\text{diag}(a){Q}\text{diag}(b)$. 
      \ENDFOR
       \STATE Update $W$ via Equation~\ref{eq:prox} in the main paper.  
       \ENDFOR
\end{algorithmic}
\end{algorithm}

\subsubsection{\model Training} \label{subapp:gtuning}
In Section 4.3, we propose to approximate the graphon is to learn a mapping function graph structure $A$ and node embedding $H$ to the target $W$.
Building upon theoretical derivation of Theorem~\ref{theo:taylor}, we have developed a noval graphon reconstruction method based on the innate characteristics of the graphon.

\begin{algorithm}[h]
    \caption{Pseudo Code for Training \model}	
    \label{alg:model}
    \begin{flushleft}
    \textbf{Input:} The pre-trained GNN $\Phi$, the shallow layer $f_{\phi}$, the \Wencoder $\Psi$, downstream graph structure and features $\{A_i, X_i\}$, the estimated oracle graphon $W$. \\
    \textbf{Output:} The fine-tuned GNN $\Phi$.
    \end{flushleft}
\begin{algorithmic}[1]
    \WHILE{$\mathcal{L}$ not converges}
    \STATE Use a pre-trained GNN $\Phi$ to obtain latent representations $H$: \quad $H = \Phi(A, X)$;
    \STATE Calculate the task-specific loss $\mathcal{L}_\text{task}$ with labels $Y$ and $H$: \quad $\mathcal{L}_\text{task} = \mathcal{L}_\text{CE}(f_{\phi}(H),{Y})$;
    \STATE Encode H and A into coefficients $\alpha = \{\alpha_1,...,\alpha_C\}$: \quad $\alpha = \Psi(A, H)$;
    \STATE Randomly select downstream graphs $G_k$, \\ $\mu_A^k=\text{sort}(G_k), \quad \boldsymbol{\B}_k = \sigma(\mu_A^k)$, \quad where $\sigma$ is the Sigmoid function;
    \STATE Compute approximated graphon: \\ $\hat{\alpha_k} = \tau( W^T [ \alpha_k || \boldsymbol{0}]), \quad \hat{\boldsymbol{\B}_k} = \text{interp1d}_{M}(\boldsymbol{\B}_k), \quad \hat{W} = \sum^{C}_{k=0} \boldsymbol{\hat{\alpha}}_k \hat{\boldsymbol{\B}_k}$, \\ where $\text{interp1d}_{M}(\cdot)$ samples $M$ values from the input vector via linear interpolation.
    \STATE Compute batch loss: \\ $\mathcal{L} = \mathcal{L}_{\text{task}} + \lambda \mathcal{L}_{\text{G-TUNNING}}(W, \hat{W}), \quad \theta_{\Psi} \xleftarrow{+} - \nabla_{\Psi} \mathcal{L}, \quad  \theta_{f_{\Phi}} \xleftarrow{+} - \nabla_{f_{\Phi}}\mathcal{L}$
    \ENDWHILE
\end{algorithmic}
\end{algorithm}

\clearpage

\subsection{Details of Experiments} \label{app:settings}
\vpara{Experiments Settings}
In Sec~\ref{sec:exp}, to separately validate the effectiveness (Q1) and transferability (Q2) of our framework, we employ two settings: The first setting (Sec~\ref{sec:Q1}) uses graph data from the same domain to pre-train and fine-tune. In this scenario, there is a certain gap between the pre-training and downstream graph structures, but the node features are in the same space. Therefore, we can directly use the original node features of the graphs. The second setting (Sec~\ref{sec:Q2}) involves pre-training and fine-tuning graphs from different domains. Here, the gap in graph structures is more significant, and the spaces of the node features are entirely different. Consequently, a common approach, GCC~\cite{qiu2020gcc}, is to completely discard the original node features.
Specifically, GCC focuses purely on structural representations without input features, because transfer learning between graph with different features has not been realized and lacks a suitable intuitive explanation. 
Since most GNNs require node features as input, GCC proposes leveraging the graph structure to initialize the node features.

Therefore, the performance reported in Tab 2 is produced using the structural node features generated in the same way as GCC rather than using the original node features.
Since we use the different features from ~\cite{Errica2020A}, the performance is naturally different from the one shown in~\cite{Errica2020A}. 

It is worth noting that our method can be leveraged to ``supervised'' setting with any other node features. 
As pointed by the reviewer, we would like to provide the results of applying \model on features used in~\cite{Errica2020A}.
Specifically, we supplement an experiment using GIN with \model without pre-training on ENZYMES and PROTEINS using their original node features.
The performance shown in Tab~\ref{tab:my_label} indicates that our method can improve the downstream task even without pre-training, which is consistent with the conclusion in the paper.

\subsubsection{Hyperparemeter Strategy}
We use Cross-entropy loss and MSE loss for graph classification and graph regression respectively.
For training process, we use Adam optimizer with early stopping at 20 epochs to train \model. 
Moreover, other hyper-parameters are searched using random search strategy and the range of hyper-parameters are listed in Tab~\ref{tab:para}, where $C$ is the number of graphon bases, $K$ is the size of oracle graphon and $M$ is the size of learnable graphon bases.

\begin{table}[h]

\begin{minipage}{0.43\linewidth}
    \centering
\resizebox{1.0\linewidth}{!}{
    \begin{tabular}{c|c|c} \toprule
         &  PROTEINS  & ENZYMES   \\ \midrule
    GIN     &  73.3±4.0 & 59.6±4.5 \\
    GIN with \model  &  \textbf{74.9±3.1} & \textbf{66.0±6.2}  \\ \bottomrule
    \end{tabular}
    }
    \caption{Mean and standard deviation of Accuracy(\%) comparison of GIN with and without \model on both datasets with their original node features. }
    \label{tab:my_label}
\end{minipage}
\hfill
\begin{minipage}{0.55\linewidth} 
\centering
\resizebox{1.0 \linewidth}{!}{
\begin{tabular}{l|c}
\toprule
        Hyper-parameter        & Range               \\ 
\midrule
     $\lambda$  & \{5e-2 $\to$ 6\}                  \\
     $C$        & \{2,4,8,16,32,64,128,256,512\}    \\
     $K$        & \{20,50,100,150,200,400,600,800,1000\}       \\
     $M$        & \{10,20,30,40,50,60,70,80,90,100\} \\
     Learning Rate  & \{1e-3 $\to$ 5e-1\}  \\
     Weight decay   & \{1e-2,1e-3,1e-4,5e-4,1e-5,1e-6,1e-7,0\}  \\
     Dropout rate   & \{0 $\to$ 0.8\}   \\
     Batch size     & \{16,32,64,128\}  \\
    \midrule
     Optimizer  & Adam  \\
     Epoch      & 100 (30 for ENZYMES and MSRC\_21) \\
     Early stopping patience & 20 \\
     GPU        & GeForce RTX 3090 \\
  
\bottomrule                     
\end{tabular}                           
}
\caption{Hyper-parameter search range for \model. }
\label{tab:para}
\end{minipage}

\end{table}



\subsubsection{Datasets} \label{app:data}
In this Section, we introduce the statistics of the datasets used in both observations and main experiments.

\vpara{Datasets in observations.} 
Here we introduce the datasets mentioned in Fig~\ref{fig:intro} of Section~\ref{sec:intro}. As shown in Tab~\ref{tab:obsdataset}, the scale of data used for pre-training and fine-tuning varies, and the data sources are diverse, hence leading to a generalized conclusion: the subpar performance of fine-tuning is due to the disparity in the generative patterns between the pre-training and fine-tuning graphs.

\vpara{Datasets in experiments.} The statistics of downstream datasets are summarized in Tab~8. The bold numbers suggest that our model is applicable on both small and large scale of dataset.

\begin{table}[h]

\begin{minipage}{0.55\linewidth}
    \centering
\captionof{table}{The statistics of pre-training and downstream datasets in observations of Fig~\ref{fig:intro}.}
\label{tab:obsdataset}
\resizebox{1.0\linewidth}{!}{
 {
    \setlength\tabcolsep{2pt} 
    \begin{threeparttable}
    \begin{tabular}{ccccccccccc}
     \toprule
        &  \multicolumn{1}{c}{Name}  &  \multicolumn{1}{c}{$|{V}|$}  &  \multicolumn{1}{c}{$|{E}|$}   \\
     \midrule
     \multirow{6}*{\setlength\tabcolsep{1pt}\rotatebox{90}{\textbf{Pre-training}}} 
     &  Academia~\cite{qiu2020gcc} &  137,969  &  739,384   \\
     &  DBLP(NetRep)~\cite{qiu2020gcc} & 540,686  & 30,491,458  \\
     & IMDB~\cite{qiu2020gcc} &  896,305  &  3,782,447 \\
     &  Michigan~\cite{TRAUD20124165}  & 30,106 & 1,176,489 \\
     &  UIllinios~\cite{TRAUD20124165} &  30,795 & 1,264,421  \\
     &  MSU~\cite{TRAUD20124165} &  32,361 & 1,118,767  \\      
     &  Wiki-Vote~\cite{snapnets} & 	7,115 & 103,689  \\
     \midrule
     \multirow{8}*{\rotatebox{90}{\textbf{Downstream}}} 
     & Pubmed~\cite{yang2016revisiting} & 19,717 & 44,324\\
     & Cora~\cite{yang2016revisiting} & 2,485  & 5,069   \\ 
     & Wisconsin~\cite{2020Geom} & 251 &  466 \\
     & Texas~\cite{2020Geom} & 183 & 1074 \\
     & Cornell~\cite{2020Geom} & 183 & 280 \\
     & DD68~\cite{nr15aaai} & 775 & 2,093 \\
     & DD687~\cite{nr15aaai} & 725 & 2,600 \\
     & DD242~\cite{nr15aaai} & 1,284 & 3,303 \\
    \bottomrule
    \end{tabular}
    \end{threeparttable}
}
}
\end{minipage}%
 \hfill
\begin{minipage}{0.43\linewidth}
    \centering
\label{tab:statistic}
\captionof{table}{The statistics of the downstream datasets in Section~\ref{sec:exp}.}
        \resizebox{1.0\linewidth}{!}{
    \begin{tabular} 
        {l|r|rr}
        \toprule
        dataset  & \# Graph & \# Avg. Node & \# Avg. Edge \\
        \midrule
        BBBP     & 2039 & 24.06 & 51.90 \\
        Tox21     & 7831 & 18.57 & 38.58 \\
        Toxcast  & 8575 & 18.78 & 38.52 \\
        SIDER     & 1427 & 33.64 & 70.71 \\
        ClinTox   & 1478 & 26.15 & 55.76 \\
        MUV      & \textbf{93087}& 24.23 & 52.55 \\
        HIV      & 41127& 24.51 & 54.93 \\
        BACE      & 1513 & 34.08 & 73.71 \\ \midrule
        IMDB-M  & 1500 & 19.77 & 96.53 \\
        IMDB-B  & 1000 & 13.00 & 65.94 \\
        MUTAG    & 188  & 17.93 & 19.79  \\
        PROTEINS & 1113 & 39.06 & 72.82 \\
        ENZYMES  & 600  & 32.63 & 62.14  \\
        MSRC\_21 & 563  & 77.52 & 198.32 \\
        RDT-M12K & \textbf{11929} &  391.41	& 456.89 \\
        \bottomrule
\end{tabular}
}
\end{minipage}
\end{table}

\clearpage

\subsection{Additional Experiments}

\subsubsection{Multi-task Fine-tuning}
Fine-tuning is originally proposed to transfer knowledge from pre-training and improve performance on downstream tasks. However, when conduct fine-tuning on multi-task, there can be a negative transfer, which means the performance is worse than it would be without pre-training.
We analyze performance on 12 tasks of Tox21 separately. 
Although fine-tuning improves the performance on some tasks, the remaining ones (see Fig~\ref{fig:negative}) performs even worse than those without pre-training.
Interestingly, \model shows improvements in almost all tasks. It confirms that only the essence of graph structure, i.e. generative patterns are universal across multiple tasks. Therefore, preserving these structural patterns is universally beneficial to all tasks.
\begin{figure}[h!]
    \begin{minipage} {0.53\linewidth}
    \centering
    \includegraphics[width=0.95\linewidth]{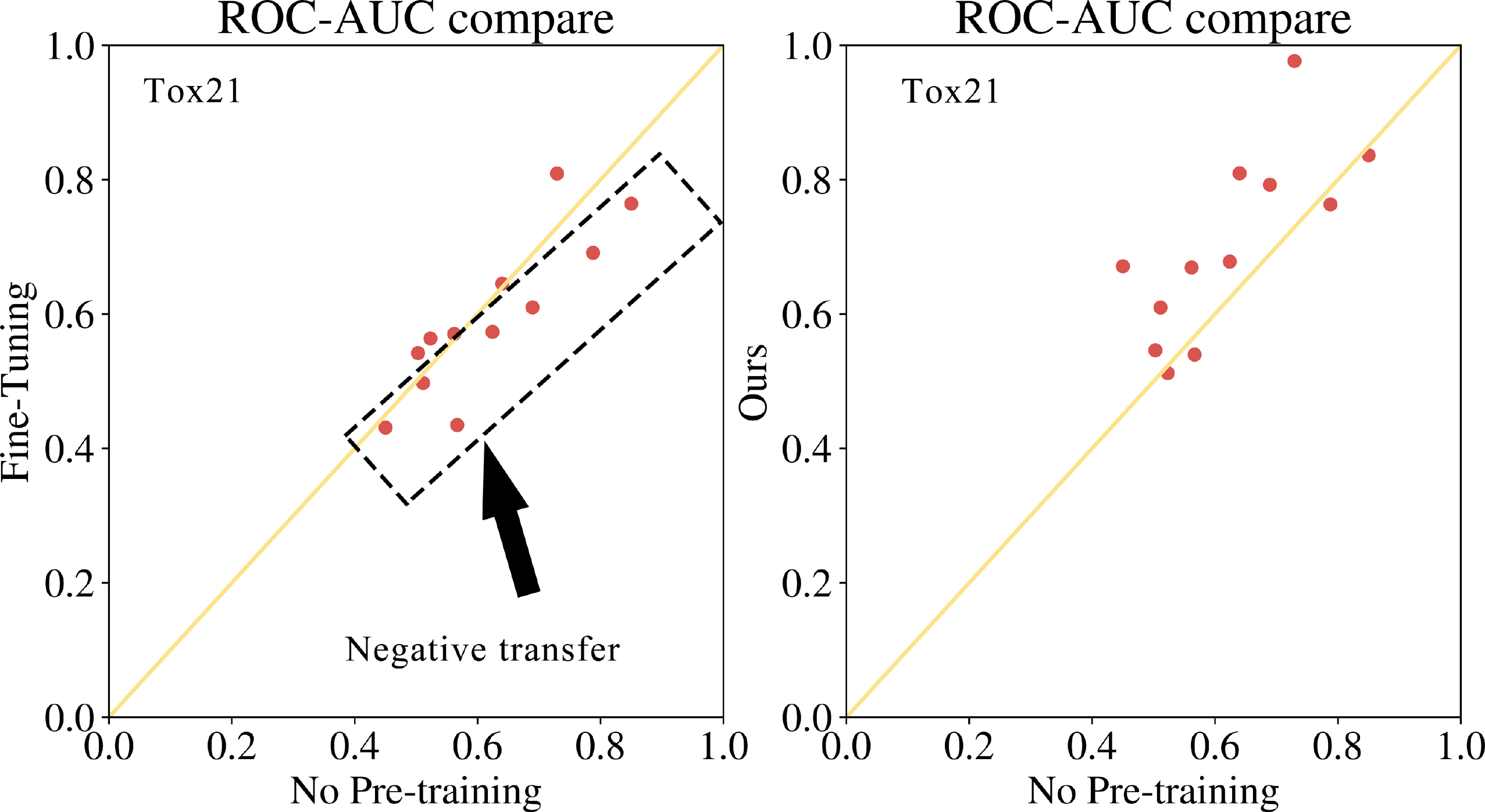}
    \setlength{\belowcaptionskip}{-0.5cm} 
    \caption{Performance comparisons between \model and vanilla tuning of multi-task dataset Tox21 versus the situation without pre-training. The dots under the yellow line indicate negative transfer.}
    \label{fig:negative}
    \end{minipage}
    \hfill
    \begin{minipage}{0.45\linewidth}
    \centering
    \includegraphics[width=0.90\textwidth]{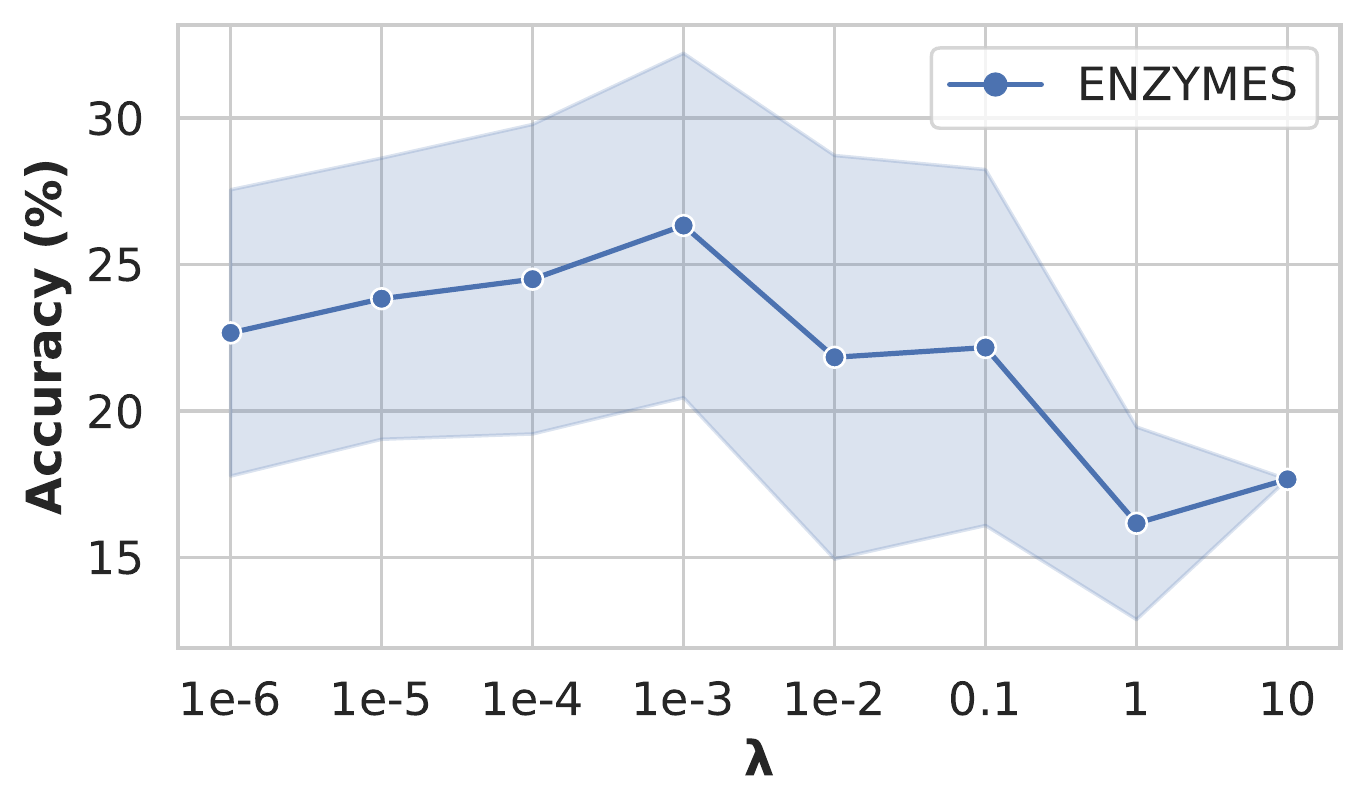}
    \caption{Performance versus magnitude of tradeoff parameter $\lambda$.}
    \label{fig:lambda}
    \end{minipage}
\end{figure}

\subsubsection{Hyper-parameter Sensitivity}

Without loss of generality, we show one dataset for each additional experiments on hyper-parameters.

 \vpara{Sensitivity of~$\lambda$} 
 As shown in Fig~\ref{fig:lambda}, the performance are improved stably compared with baselines as the $\lambda$ grows before $1e-3$. However, since the larger the $\lambda$ is, the more likely the downstream task is to be neglected. Thus, an excessive increase in $\lambda$ will lead to a decrease in effectiveness.
 
\vpara{Sensitivity of $M$}
As shown in Fig~7, the performance and the training time does not change much when changing $M$. The reason is mainly because when initializing learnable bases, we use the original structure of $C$ downstream graphs, and randomly sample from these matrices to same size $M$. 

\vpara{Sensitivity of $D$}
As shown in Fig~8, the performance benefits from a certain range of $D$. Neither with too large nor too small $D$, it is difficult for \model to grasp the essence of generative patterns and presents a relatively large variance due to the difficulty of convergence. The reason is that the estimation process effectively condense the generative patterns into the oracle graphon. Besides, the GW-divergency based loss is able to transfer these information to the predicted graphon and guide the pre-trained GNN to adapt to downstream graphs.

\begin{figure}[htbp]
  \centering
  \begin{minipage}[t]{0.48\linewidth}
    \centering
      \includegraphics[width=\linewidth]{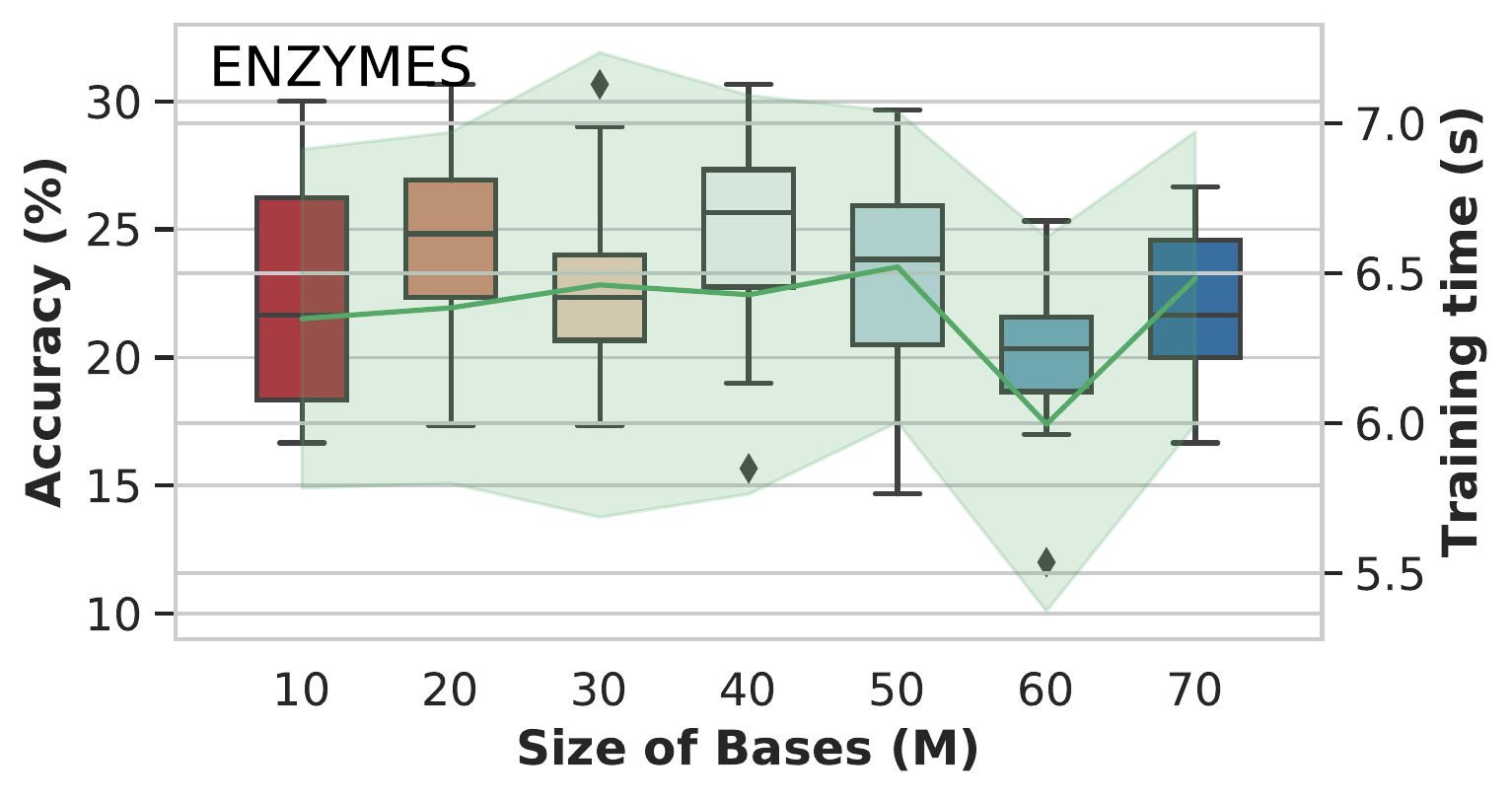}
    \label{fig:sizeM}
    \caption{Performance versus the size of bases $M$.}
  \end{minipage}
  \hfill
  \begin{minipage}[t]{0.48\linewidth}
    \centering
      \includegraphics[width=\linewidth]{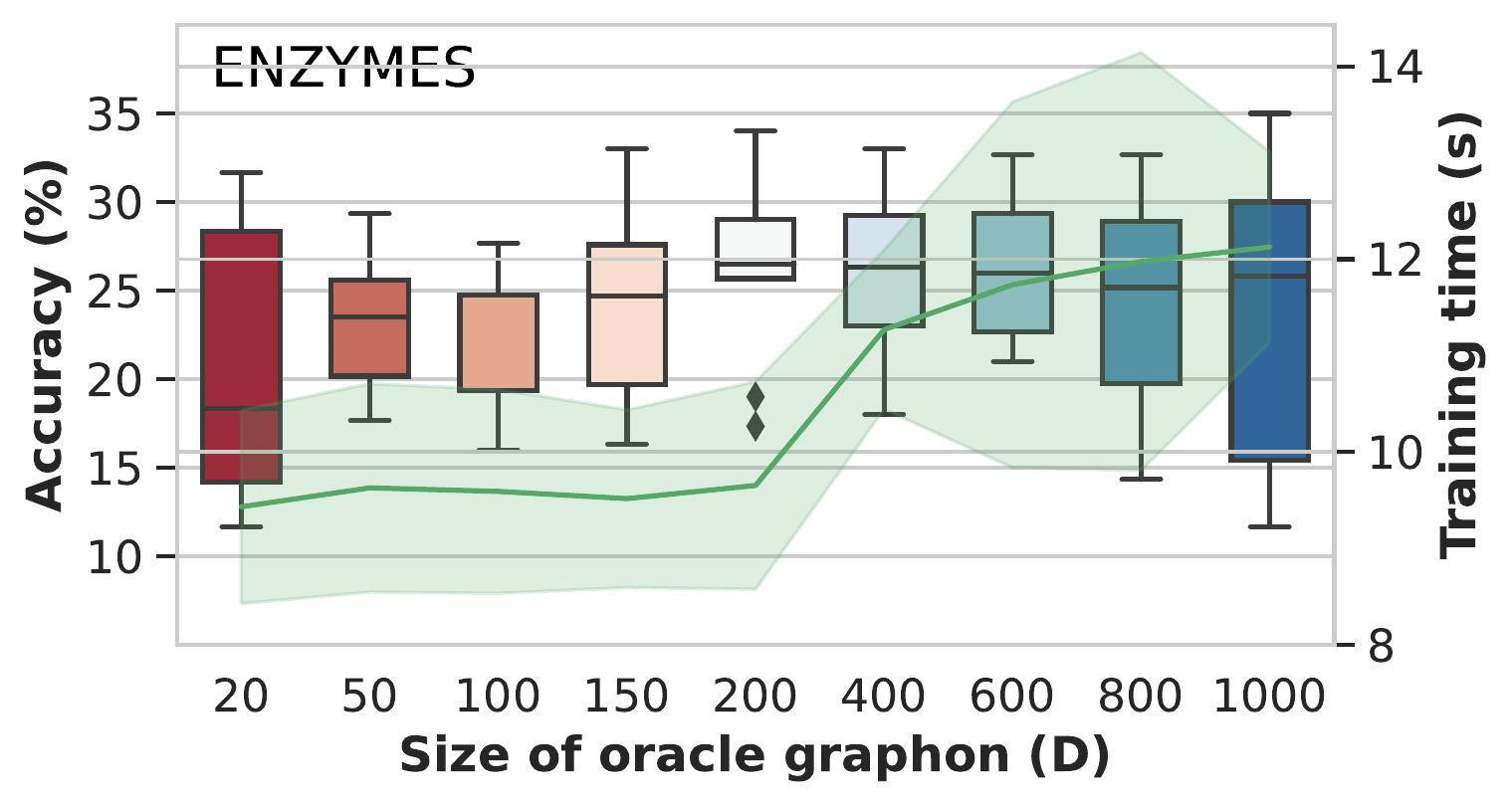}
    \label{fig:sensitivity of D}
    \caption{Performance versus the size of oracle graphon $D$.}
  \end{minipage}
  \label{fig:sizeD}
\end{figure}

%% file: 6.5-related.tex
In this section, we supplement with some more work about graph pre-training and graphon estimation.

\vpara{Pre-training GNNs.}
With the growing emphasis on graph data in various applications such as finance~\cite{chai2023towards,huang2022dgraph}, neuroscience~\cite{chen2022brainnet} and social network analysis~\cite{yang2022s}, Graph Neural Networks (GNNs) have gradually become one of the hottest neural network models. Additionally, with the rise of large-scale models, the paradigm of pre-training graph foundation model~\cite{liu2023towards} on extensive graph data and fine-tuning them for downstream tasks has once again become a prevailing choice. However, most of the GNNs are not capable of handling different types of graph data. Therefore, some work~\cite{sun2022beyond,mao2023demystifying} is also attempting to break the limitations of the original  and design graph model architectures that can be applied to a wider variety of data types. 

The intuition behind the pre-training on graphs is that there exists underlying transferable patterns between pre-train and downstream graphs~\cite{qiu2020gcc}, such as scale-free~\cite{albert2002statistical}, small world~\cite{watts1998collective} and motif distribution~\cite{milo2004superfamilies}. 
Existing methods mainly design various self-supervised tasks for pre-training GNNs, \emph{e.g.}, context prediction~\cite{hu2020strategies,rong2020self}, edge/attribute generation~\cite{hu2020strategies}, graph autoencoder~\cite{hou2022graphmae} and contrastive learning~\cite{you2020graph,qiu2020gcc}, which are proved to be effective when adapted on similar graphs as pre-train and downstream datasets. 

Generally, most of them simply use vanilla fine-tuning strategy, namely, the pre-trained GNN directly serves as weight initialization for downstream dataset.~\cite{xia2022survey}
However, due to the large structural divergence between pre-train and downstream datasets, the performance is not optimal and even negative transfer may occur~\cite{xia2022survey,hu2020strategies}. 
Thus, we focus on fine-tuning stage to better adapt pre-trained model to downstream graphs rather than designing new pre-training methods. 
We introduces a novel approach, \model, to enhance the generalization of a pre-trained model during the fine-tuning stage, aiming to achieve improved performance on downstream tasks.
Furthermore, \model can be applied to improve the fine-tuning performance of any existing GNN pre-training algorithms.

\vpara{Graphon Estimation.}
Existing research leverages the graphon as a non-parametric method to describe and estimate large networks~\cite{borgs2017graphons}. 
Estimating graphons from graphs is a prerequisite for research based on graphons.
As described in Sec~\ref{sec:prelim}, the weak regularity lemma of graphon~\cite{frieze1999quick} guarentees that an arbitrary graphon can be approximated well by a two-dimensional step function. 
Thus, graphon estimation is essentially to estimate step function.
The existing methods generally align all nodes in multiple graphs or multiple subgraphs based on some measurement, and then estimate the graphon based on the aligned matrices. On the one hand, there are some work learn low-rank matrices to approximate graphons, like matrix completion (MC)~\cite{keshavan2010matrix} and universal singular value thresholding (USVT)~\cite{chatterjee2015matrix}. 
On the other hand, other work base on stochastic block models. A basic method of this kind is the stochastic block approximation (SBA)~\cite{airoldi2013stochastic}.
Largest gap (LG)~\cite{channarond2012classification} improve the SBA to be able to be applied on both large and small scale of graphs. Sorting and smoothing (SAS)~\cite{chan2014consistent} method add a total variation minimization step to sort the graphs.
The above methods can only estimate the graphon base on the graphs come from a single graphon. \citeauthor{xu2021learning} develop a structured Gromov-Wasserstein Barycenters (GWB) based learning method which efficiently estimate the graphon to cope with the graphs from real world. 
These estimation methods can all serve as tools for preserving generative patterns. In the future, if there emerge methods that offer better and faster performance, our approach will also undergo further enhancement and acceleration.
